\documentclass{article}

\usepackage[preprint]{neurips_2025}

\usepackage{url}
\usepackage{algorithm}
\usepackage{algorithmic}
\usepackage{amsthm}
\usepackage{amsmath}
\usepackage{amssymb}
\usepackage[T1]{fontenc}

\theoremstyle{plain}
\newtheorem{theorem}{Theorem}[section]

\newtheorem{lemma}[theorem]{Lemma}

\theoremstyle{definition}
\newtheorem{definition}[theorem]{Definition}

\theoremstyle{remark}

\DeclareMathOperator*{\argmax}{arg\,max}
\DeclareMathOperator*{\argmin}{arg\,min}

\usepackage[utf8]{inputenc} % allow utf-8 input
\usepackage[T1]{fontenc}    % use 8-bit T1 fonts
\usepackage{hyperref}       % hyperlinks
\usepackage{url}            % simple URL typesetting
\usepackage{booktabs}       % professional-quality tables
\usepackage{amsfonts}       % blackboard math symbols
\usepackage{nicefrac}       % compact symbols for 1/2, etc.
\usepackage{microtype}      % microtypography
\usepackage{xcolor}         % colors
\usepackage{multirow}
\usepackage{subcaption}
\usepackage{graphicx}

\title{Attributing Data for Sharpness-Aware Minimization}

\author{
Chenyang Ren\textsuperscript{*,1,2},
Yifan Jia\textsuperscript{*,1,3},
Huanyi Xie\textsuperscript{*,1,2},
Zhaobin Xu\textsuperscript{1,3},\\
\textbf{Tianxing Wei\textsuperscript{1,3},
Liangyu Wang\textsuperscript{1,2},
Lijie Hu\textsuperscript{†,1,2},
Di Wang\textsuperscript{†,1,2}}\\
$^1$Provable Responsible AI and Data Analytics (PRADA) Lab\\
$^2$King Abdullah University of Science and Technology \\
$^3$Shandong University %\quad $^4$MBZUAI
}

\begin{document}
\maketitle

\begin{abstract}
Sharpness-aware Minimization (SAM) improves generalization in large-scale model training by linking loss landscape geometry to generalization. However, challenges such as mislabeled noisy data and privacy concerns have emerged as significant issues. Data attribution, which identifies the contributions of specific training samples, offers a promising solution. However, directly rendering existing data influence evaluation tools such as influence functions (IF) to SAM will be inapplicable or inaccurate as SAM utilizes an inner loop to find model perturbations that maximize loss, which the outer loop then minimizes, resulting in a doubled computational structure. Additionally, this bilevel structure complicates the modeling of data influence on the parameters. In this paper, based on the IF, we develop two innovative data valuation methods for SAM, each offering unique benefits in different scenarios: the Hessian-based IF and the Gradient Trajectory-based IF. The first one provides a comprehensive estimation of data influence using a closed-form measure that relies only on the trained model weights. In contrast, the other IF for SAM utilizes gradient trajectory information during training for more accurate and efficient data assessment. Extensive experiments demonstrate their effectiveness in data evaluation and parameter tuning, with applications in identifying mislabeled data, model editing, and enhancing interpretability.
\end{abstract}

\def\thefootnote{*}\footnotetext{Equal Contribution.}
\def\thefootnote{†}\footnotetext{Corresponding Author.}

\section{Introduction}
Over the past decade, deep neural networks have advanced significantly due to increased model parameter sizes and improved training algorithms that enhance generalization. However, larger models often memorize training data, leading to overfitting and poor generalization. In order to address this issue, considerable effort has been invested in the development of a range of strategies, including differential privacy~\cite{hu2023differentially,hu2023privacy,wang2023generalized,zhang2025improved,xiang2024preserving,xiao2023theory,xiang2023practical,zhang2025towards}, regularization techniques~\citep{wu2021r, yoshida2017spectral,luo2025privacy} and adversarial training~\citep{mkadry2017towards, shafahi2019adversarial,fu2023theoretical,fu2025short,liu2025taad,zheng2020towards}. Recent work has observed that sharp local minima in the loss landscape can significantly impair the generalization performance of deep networks~\citep{keskar2016large, hochreiter1994simplifying, neyshabur2017exploring}. To make the loss landscape flatter to improve models' generalization ability, ~\cite{foret2020sharpness} introduced a general framework called Sharpness-aware Minimization (SAM). SAM improves generalization by penalizing sharp minima and encouraging convergence to flatter regions. Intuitively, SAM is a bilevel optimization problem, where the inner level seeks weight perturbations that can lead to the maximum loss, which is a measure of local sharpness. On the outer level, the model is trained to minimize both loss and local sharpness simultaneously. Thus, it can also be formulated as a minimax optimization problem. SAM has achieved state-of-the-art results across various tasks~\citep{foret2020sharpness, chen2021vision, liu2022towards}.

While SAM has been applied in various real-world applications~\citep{du2021efficient, andriushchenko2023sharpness, qu2022generalized, bahri2021sharpness}, the presence of noisy data in the training set, including mislabeled or poisoned data, has become a significant concern. A critical approach to tackling this issue involves identifying the contributions of training samples in SAM-trained models by assessing their impact on model performance, a process referred to as (training data) attribution. Data attribution plays a critical role in tracing model outputs back to significant training examples, thereby providing insights into how individual data points influence model performance. 

Data attribution has been used in various tasks to enhance model transparency and understanding of how training data influences model behavior~\citep{ribeiro2016should,lundberg2017unified,zhang2025mechanistic}. Generally, data attribution methods~\citep{jia2019towards,ghorbani2019data,yoon2020data,han2020explaining} assign higher contribution scores to training instances that significantly improve model performance when included, which can be divided into two types. The first one, such as Shapley Value~\citep{winter2002shapley}, is based on sampling~\citep{lundberg2017unified, pmlr-v151-kwon22a}, which requires multiple retraining with different data subsets. This is computationally expensive and impractical for large models. To address this challenge, the second approach—namely, influence function-based methods~\citep{koh2017understanding, feldman2020neural,ren2025evaluating,hu2024editable,hu2024dissecting}—estimates data contributions using gradient information, thereby facilitating accurate assessments without the need for retraining. 
Recent advancements have led to innovative estimators utilizing the training gradient trajectory~\citep{pruthi2020estimating, schioppa2024theoretical}, which estimate the influence of training data on model predictions by tracing gradient descent, providing better insights into and improvements of training examples in deep learning models. These methods do not depend on the convexity assumption of the loss function, nor do they require computationally intensive Hessian inversions. Instead, they perform calculations based on the training gradient and have achieved good results in many tasks~\citep{liu2024tracing}. 

 However, it is essential to note that influence functions were initially designed for M-estimators~\cite {huber1981robust} (i.e., minimizing the summation of loss on all data), which limits their direct applicability to the bilevel structure in SAM. Specifically,  the influence of data on perturbations, and consequently on model parameters, is in an indirect manner and has not been considered in classical influence functions. Due to the coupled outer-inner optimization process involved in SAM, any changes to the model parameters will necessitate further updates to the inner perturbations and consequently alter the outer model parameters. Furthermore, in the actual training of SAM, the inner perturbations are not derived through the optimization process but are instead via approximation. To further accelerate computations, the outer gradient descent is also approximated, resulting in minor deviations in the inner perturbations from the original SAM model. These complicate our ability to evaluate the influence of data in SAM accurately.

In this paper, we propose two data valuation methods based on influence functions for SAM, focusing on two distinct scenarios. Firstly, we derive a closed form of the Hessian-based Influence Function (SAM-HIF) from the mathematical formulation of SAM through theoretical derivation. SAM-HIF fully considers the impact of data on the model's local sharpness. In the absence of gradient trajectory information, SAM-HIF serves as a comprehensive estimation method. However, when gradient trajectory information during training is accessible, we can focus on the specific training algorithm to derive a more accurate evaluation method. Inspired by \cite{pruthi2020estimating}, we propose the Gradient Trajectory-based Influence Function (SAM-GIF). This method utilizes checkpoints to evaluate the influence of data within SAM, leading to accurate data assessment while reducing implementation and computational costs. SAM-GIF proves to be highly effective due to its simplicity, high scalability, and accuracy.
To sum up, our contributions are listed as follows:
\begin{itemize}
    \item We present data evaluation methods using influence functions (IF) for SAM and its training algorithms.  First, we derive SAM-HIF, which allows for precise modeling of data influence in SAM by effectively capturing influences on both model parameters and model perturbations.  Second, we present SAM-GIF for SAM training algorithms, which showcases exceptional precision and scalability when training gradient trajectories are accessible. We also discuss some downstream applications and strategies to accelerate the computation of our influence functions.
    
   \item Our method has potential applications in several real-world scenarios, including the automatic identification and removal of potentially harmful data points, SAM model editing, and enhancing the interpretability of SAM-trained models. We validate the effectiveness and efficiency of SAM-HIF and SAM-GIF through comprehensive experimental evaluations. Experimental results indicate that our framework performs well in data evaluation and the tuning of SAM-trained models.
\end{itemize}

\section{Related Work}
\paragraph{Sharpness Aware Minimization.}
The concept of flat minima and its link to reducing overfitting and enhancing model generalization was first explored by \cite{hochreiter1997flat}, establishing a foundation for understanding why models that converge to flatter regions tend to generalize better. Additionally, \cite{keskar2016large} experimentally examined the relationship between batch size, sharp minima, and generalization. Their findings provided empirical support for SAM's theoretical foundations and highlighted the importance of sharpness in optimization and model training. Overall, sharp local minima can significantly influence the generalization capabilities of deep networks~\citep{chaudhari2019entropy, izmailov2018averaging}. Building on these insights, one of the foundational works on SAM is by~\cite{foret2020sharpness}, who proposed a bi-level optimization framework that seeks perturbations around model parameters to maximize loss and guide the model towards flatter minima.  Recent studies have further explored the behavior SAM. For example,~\cite{andriushchenko2022towards} investigated the geometric properties of loss landscapes, demonstrating that SAM's perturbation-based approach offers insights into model behavior. Their findings indicate that SAM not only improves generalization but also enhances understanding of models' sensitivity to input perturbations.~\cite{zhao2022penalizing} assert that SAM~\citep{foret2020sharpness} is effectively equivalent to applying a gradient norm regularization through the approximation of the Hessian matrix.~\cite{kwon2021asam} introduce adaptive sharpness-aware minimization, which can dynamically adjust the maximization region based on the weight scale. To minimize computational costs associated with SAM,~\cite{du2021efficient} introduced Efficient SAM, which randomly computes perturbations.

\noindent {\bf Influence Function.} 
The influence function(IF), initially developed in robust statistics \citep{cook2000detection,cook1980characterizations}, has become essential in machine learning since its introduction by \cite{koh2017understanding}. IFs have been used to various fields, including interpreting model outputs, reducing model bias~\citep{wang2019repairing}, and facilitating machine unlearning~\citep{liu2024certified, golatkar2020eternal,golatkar2021mixed}. 
Its versatility spans various fields, including natural language processing \citep{han2020explaining} and image classification \citep{basu2021influence}, while also addressing biases in classification models \citep{wang2019repairing}, word embeddings \citep{brunet2019understanding}, and model fine-tuning \citep{chen2020multi}. A recent innovative influence function that leverages the training gradient trajectory has been proposed and investigated~\citep{pruthi2020estimating, schioppa2024theoretical}, and has been successfully applied in instructional fine-tuning~\citep{xia2024less}. Despite the numerous studies related to influence functions, we are the first to use this concept for the data evaluation of models trained for SAM. This provides us with a completely new perspective on understanding the training process of SAM.

\section{Preliminaries}
\noindent{\bf Sharpness-Aware Minimization (SAM).}
SAM presents a novel approach to optimizing machine learning models with the primary objective of enhancing generalization by mitigating the sharpness of the loss landscape. Given a dataset $S = \{(x_i, y_i)\}_{i=1}^n$, we define the training loss associated with model parameters $\omega$ and a perturbation $\epsilon$ as:
$$
    L_S(\omega + \epsilon) = \frac{1}{n} \sum_{i=1}^n\ell(x_i,y_i;\omega + \epsilon) \triangleq  \sum_{i=1}^n L_S^i(\omega+{\epsilon}).
$$
Then SAM optimization framework seeks parameters that lie
in neighborhoods having uniformly low loss by the following procedure. Firstly, seeking the model perturbation maximizing the loss:
\begin{equation*}
   \hat{\epsilon}(\omega) = \argmax_{||\epsilon||_p\leq\rho} L_S(\omega+ \epsilon).
\end{equation*}
Then minimize the uniform loss by the following loss function:
\begin{equation}\label{initial:para}
    \omega^* =  \argmin_{\omega}L_S(\omega + \hat{\epsilon}(\omega)) + \frac{\lambda}{2}\cdot ||\omega||^2_2.
\end{equation}
where $\lambda$ is the regularization parameter.
Here, SAM takes model loss sharpness into consideration via including the perturbation process. The loss function for SAM can be formed into one equation as follows:
\begin{equation*}
\begin{split}
        L^{SAM}_S(\omega) =&  L_S(\omega+ \hat{\epsilon}(\omega))+ \frac{\lambda}{2}\cdot ||\omega||^2_2 \\
        =&  \max_{||\epsilon||_p\leq\rho} L_S(\omega+ \epsilon) + \frac{\lambda}{2}\cdot ||\omega||^2_2.
\end{split}
\end{equation*}
% Besides, $\hat{\epsilon}(\omega)$ is defined as
% \begin{equation*}
%    \hat{\epsilon}(\omega) = \argmax_{||\epsilon||_p\leq\rho} L_S(\omega+ \epsilon).
% \end{equation*}
% Based on this, $L^{SAM}_S(\omega)$ can be formed as a bi-level problem:
% \begin{equation*}
%     L^{SAM}_S(\omega) = L_S(\omega+ \hat{\epsilon}(\omega)) = \max_{||\epsilon||_p\leq\rho}  \sum_{i=1}^n L_S^i(\omega+{\epsilon})
% \end{equation*}
In essence, SAM seeks to minimize the worst-case loss by enforcing a level of robustness against perturbations in the parameter space, thereby promoting smoother loss landscapes. 
Besides, the gradient of $L_{S}^{SAM}$ is calculated by
\begin{equation}\label{m1:gradient}
    \nabla L_S^{SAM}(\omega) = \nabla  L_S(\omega + \hat{\epsilon}(\omega)) + \frac{\mathrm{d}\hat{\epsilon}(\omega)}{\mathrm{d}\omega}\cdot \nabla  L_S(\omega + \hat{\epsilon}(\omega)).
\end{equation}
% And the Hessian matrix is 
% \begin{equation}\label{m1:hessian}
%     \nabla^2 L_S^{SAM}(\omega) = \nabla^2  L_S(\omega + \hat{\epsilon}(\omega)) + \frac{\mathrm{d}^2\hat{\epsilon}(\omega)}{\mathrm{d}\omega^2}\cdot \nabla  L_S(\omega + \hat{\epsilon}(\omega)) +\frac{\mathrm{d}\hat{\epsilon}(\omega)}{\mathrm{d}\omega}\cdot \nabla  L_S(\omega + \hat{\epsilon}(\omega)) 
% \end{equation}

\noindent {\bf Influence Function.} 
The influence function~\citep{huber1981robust} quantifies how an estimator relies on the value of each individual point in the sample. Consider a neural network $\hat{\theta}= \argmin_\theta L(\theta, D)=\sum_{i=1}^n \ell(z_i; \theta)$ with loss function $\ell$ and dataset $D=\{z_i\}_{i=1}^n$. When an individual data point $z_m$ is removed from the training set, the retrained optimal retrained model is denoted as $\hat{\theta}_{-z_m}$. The influence function method provides an efficient way to approximate $\hat{\theta}_{-z_m}$ without the need of retraining. By up-weighing $z_m$-related term in the loss function by $\delta$, a series of $\delta$-parameterized optimal models $\hat{\theta}_{-z_m, \delta}$ will be obtained by 
$$
    \hat{\theta}_{-z_m, \delta} = \argmin_{\theta}\left[L(\theta, D) + \delta \cdot \ell(z_m; \theta)\right].
$$
Consider the term 
\begin{equation}\label{eq:1}
    \nabla L( \hat{\theta}_{-z_m, \delta} , D)+ \delta\cdot \nabla\ell(z_m; \hat{\theta}_{-z_m, \delta} )=0,
\end{equation}
     we perform a Taylor expansion at $\hat{\theta}$ and incorporate the optimal gradient condition at $\hat{\theta}_{-z_m}$ and $\hat{\theta}$:
$$
\sum_{i=1}^{n} \nabla \ell(z_{i}; \hat{\theta} ) + \delta\cdot \nabla \ell(z_m; \hat{\theta}) + H_{\hat{\theta}} \cdot \left(\hat{\theta}_{-z_m,\delta} - \hat{\theta}\right) \approx 0
$$
where ${H}_{\hat{\theta}} = \sum_{i=1}^n\nabla_{\hat{\theta}}^2 \ell(z_i; {\hat{\theta}})$ is the Hessian matrix. %, with $\delta$ as a small constant to make the matrix invertible. 
Consequently, the Influence Function is defined as the derivative of the change in parameters of the retrained model due to perturbation with respect to the perturbation:
$$
\text{IF}(z_m) = \left. \frac{\mathrm{d}\hat{\theta}_{-z_m, \delta}-\hat{\theta}}{\mathrm{d}\delta}\right|_{\delta=0}  \approx -{H}_{\hat{\theta}}^{-1} \cdot\nabla \ell (z_m;\hat{\theta}).
$$
% The first term equals 00 due to the optimal gradient condition at ˆθ\hat{\theta}. Then an estimator for the parameters ˆθ−zm,ϵ\hat{\theta}_{-z_m, \epsilon} trained without zmz_m is obtained by Taylor expansion and optimal condition: 
% $$
% \hat{\theta}_{-z_m, \epsilon}  -\hat{\theta}= -\epsilon\cdot{H}_{\hat{\theta}}^{-1} \cdot \nabla  \ell (\hat{\theta}; z_m),
% $$

When setting $\delta = -1$, this results in the complete removal of $z_m$ from the retraining process. Then, $\hat{\theta}_{-z_m}$ can be approximated by a linear approximation formula as $\hat{\theta} - \text{IF}(z_m)$. %In practice, we always use a regularized Hessian ${H}_{\hat{\theta}} = \sum_{i=1}^n\nabla_{\hat{\theta}}^2 \ell(z_i; {\hat{\theta}})+ \delta I$ with some $\delta>0$ rather than the original Hessian to avoid the numerical instability issue when calculating the Hessian inversion.  
Additionally, for a differentiable evaluation function, such as one used to calculate the total model loss over a test set, the change resulting from up-weighting $\epsilon$ to $z_m$ in the evaluation results can be approximated as $-\nabla f(\hat{\theta}) \cdot \text{IF}(z_m)$.

\section{Evaluating Data Attribution in SAM }
To evaluate the influence of an individual data point for SAM, we utilize assessing model differences after leave-one-out (LOO) retraining. 
% From the definition of the parameters obtained through LOO retraining, we observe that the removal of a data point affects not only the loss function that defines the parameter search space but also the loss associated with the perturbation.
\subsection{Evaluating Data Attribution in SAM via Hessian-based IF}\label{HIF}
To evaluate the influence of an individual data point for the model trained via SAM, we first provide an estimation of the leave-one-out (LOO) retrained model. Then we can quantify the parameter-level influence of the excluded data point by analyzing the differences in model parameters before and after the retraining process. Due to the unique architecture of SAM, the form of the corresponding loss function differs from that of traditional IF. This is because SAM includes an additional step to find parameter perturbations that maximize the loss, which will also change as a result of the LOO retraining process. We provide details in the following.
\begin{definition}
[LOO Retrained Parameter]
Consider one data $(x_k,y_k)$ to be evaluated, $\hat{\epsilon}_{k}(\omega)$ denotes the perturbation around $\omega$ which results in the maximum loss after the removal of $(x_k,y_k)$, which is defined as
\begin{equation*}
    \hat{\epsilon}_{k}(\omega) = \argmax_{||\epsilon||_p\leq\rho}  \sum_{i\neq k}^n L_S^i(\omega+{\epsilon}). 
\end{equation*}
Then, the LOO retrained model is defined as 
\begin{align}\label{LOO:para}
    \omega_{k} = \argmin_{\omega} \sum_{i\neq k} L_S^i(\omega+\hat{\epsilon}_{k}(\omega)) +\frac{\lambda}{2}\cdot ||\omega||^2_2. 
\end{align}
\end{definition}
From the preceding discussion, we denote the influence of the data point $(x_k, y_k)$ on the SAM-trained model as $\omega_k - \omega^*$. To avoid retraining, we utilize the influence function method to approximate this difference. We begin by transforming the optimality condition of (\ref{initial:para}) into a simpler form via Danskin’s Theorem~\citep{danskin2012theory}.
\begin{lemma}\label{danskin}
The optimal solution $w^*$ of SAM, as represented in (\ref{initial:para}), satisfies 
\begin{equation*}
\begin{split}
\nabla  L_S(\omega^*+\hat{\epsilon}(\omega^*)) + \frac{\mathrm{d}\hat{\epsilon}(\omega^*)}{\mathrm{d}\omega^*}\cdot \nabla  L_S(\omega^*+\hat{\epsilon}(\omega^*)) + \lambda \cdot \omega^* = 0,
\end{split}
\end{equation*}
which is equivalent to the following condition
\begin{equation*}
\nabla L_S(\omega^*+\hat{\epsilon}(\omega^*)) = 0.
\end{equation*}
\end{lemma}
Follow the idea of influence function, we up-weigh the $k$-th term in the loss function by a factor of $\delta$. This adjustment allows us to derive a series of model parameters $\omega_{k, \delta}$ obtained by training the models via SAM:
\begin{align*}
   &\omega_{k, \delta}=\arg\min L_{\text{Total}, \delta}(\omega) \triangleq L_{S, \delta}(\omega)+ \frac{\lambda}{2}\cdot ||\omega||^2_2\\
    =& L_S(\omega+\hat{\epsilon}_{\delta}(\omega)) + \delta \cdot L_S^k(\omega+\hat{\epsilon}_{\delta}(\omega))+ \frac{\lambda}{2}\cdot ||\omega||^2_2
\end{align*}
The $\delta$-related term in the loss function directly reflects the influence of $(x_k,y_k)$ on $\omega$. Note that $\omega_{k, -1}=\omega_{k}$. Additionally, the learning of the worst perturbation is influenced by the up-weighting factor $\delta$, as reflected in the modified objective function for $\epsilon$:
\begin{equation*}
    \hat{\epsilon}_{k, \delta}(\omega) = \argmax_{||\epsilon||_p\leq\rho}  \left[\sum_{i=1}^n L_S^i(\omega+{\epsilon})+  \delta \cdot L_S^k(\omega+ {\epsilon}) \right]
\end{equation*}
This modification in the objective function, in turn, indirectly influences the optimal parameters learned through SAM. Starting from a simplified case, we neglect the indirect influence of the perturbation and introduce the simplified version of SAM-IF.
\begin{theorem}\label{HIF-easy}
    Consider the data $(x_k, y_k)$ along with a SAM-trained model $\omega^*$, $\hat{\epsilon}(\omega^*)$ represents the model weight perturbation that results in the largest loss.
Then the corresponding SAM-IF is defined as:
\begin{align*}
    \text{SAM-IF}(x_k, y_k) &= -H_{\omega}^{-1} \cdot\nabla L_S^k(\omega^*+ \hat{\epsilon}(\omega^*))
\end{align*}
where $H_{\omega}$ is defined as $H_{\omega} = \nabla^2 L_S\left(\omega^*+\hat{\epsilon}(\omega^*)\right)+ \lambda$. $\omega_{k}$ can be approximated by 
\begin{equation*}
    \omega_{k}  \approx  \omega^*-\text{SAM-IF}(x_k, y_k).
\end{equation*}
\end{theorem}
We can see that in SAM-IF we still use $\hat{\epsilon}(\omega^*)$ in both Hessian matrix and gradient, while the optimal perturbation should be $\hat{\epsilon}_{k, -1}(\omega_{k})$.  However, neglecting the term related to the perturbation can lead to an incomplete and inaccurate evaluation of data influence. In fact, it implies that while the influence of this data point on the model's loss function is considered, the impact of removing this data point on the computation of model sharpness is not taken into account. This will lead to an incomplete evaluation of data influence. To take this into consideration, $\hat{\epsilon}_{k, \delta}\left(\omega_{k, \delta}\right) - \hat{\epsilon}(\omega^*)$ also need to be estimated, which is given in the following lemma. 
 % Noting
% \begin{align*}
%     &\hat{\epsilon}_{\delta}(\omega_{\delta}) =\argmax_{||\epsilon||_p\leq\rho} \left[ L_{S}(\omega_{\delta}+ {\epsilon}) + \delta\cdot L_S^k(\omega_{\delta}+{\epsilon}) \right]\\
%     &\hat{\epsilon}(\omega^*) = \argmax_{||\epsilon||_p\leq\rho}
%         L_S(\omega^*+{\epsilon})
%     % & = \rho \cdot \text{sign}(\nabla_w L_S(w^*))\cdot\frac{|\nabla_w L_S(w^*)|^{q-1}}{\left( \|\nabla_w L_S(w^*)\|_q^q \right)^{1/p}}
% \end{align*}
\begin{lemma}
 The term 
$\hat{{\epsilon}}_{k, \delta}\left(\omega_{k, \delta}\right) - \hat{\epsilon}(\omega^*)$ can be approximated by
\begin{equation}\label{app:epsilon_approx}
\begin{split}
    &\hat{\epsilon}_{k,\delta}(\omega_{k,\delta}) - \hat{\epsilon}(\omega^*)\\
    =& \hat{\epsilon}_{k, \delta}(\omega^*) - \hat{\epsilon}(\omega^*)  + \left. \frac{\mathrm{d} \hat{\epsilon}_{k, \delta}(\omega)}{\mathrm{d} \omega} \right|_{\omega = \omega^*} \cdot (\omega_{k,\delta} - \omega^*).
\end{split}
\end{equation}
\end{lemma}
When $\epsilon\rightarrow 0$, $\hat{\epsilon}_{k,\delta}(\omega^*) - \hat{\epsilon}(\omega^*)\rightarrow 0$, then $\hat{\epsilon}_{k,\delta}(\omega_{k,\delta}) - \hat{\epsilon}(\omega^*)$ can be bounded by $\omega_{k,\delta} - \omega^*$, which is especially useful for the following derivation.
\begin{theorem}\label{HIF-complex}
 Consider the $k$-th data $(x_k, y_k)$ along with a SAM-trained model $\omega^*$. Define the Hessian-based influence function for SAM(SAM-HIF) as:
\begin{align*}
    &\text{SAM-HIF}(x_k, y_k)\\
    =& -\left(H_{\omega} +H_{\omega}\cdot  \frac{\mathrm{d} \hat{\epsilon}(\omega^*)}{\mathrm{d} \omega }\right)^{-1} \cdot\nabla L^k_S\left(\omega^*+\hat{\epsilon}(\omega^*)\right).
\end{align*}
 Then $ \omega_{k}$ can be approximated by:
\begin{equation*}
    \omega_{k} \approx \omega^* -\text{SAM-HIF}(x_k, y_k).
\end{equation*}
\end{theorem}
Note that in general, there is no closed-form solution for the term $ \frac{\mathrm{d} \hat{\epsilon}(\omega^*)}{\mathrm{d} \omega }$, to address the issue, we can use the method in ~\citep{foret2020sharpness} to approximate $\hat{\epsilon}$: 
\begin{equation}\label{eq:6}
    \hat{\epsilon}(\omega) = \rho \cdot \text{sign}(\nabla_w  L_{S}(\omega, {0}) )\cdot\frac{|\nabla_w  L_{S}(\omega, {0})  |^{q-1}}{\left( \|\nabla_w  L_{S}(\omega, {0}) \|_q^q \right)^{1/p}},
\end{equation}
where $p$ satisfies $\frac{1}{p}+\frac{1}{q}=1$.
 Based on the assumption that the sign of the gradient $\nabla_w  L_{S}(\omega, {0})$ remains unchanged under the removal of one data point, $\frac{\mathrm{d} \hat{\epsilon}(\omega^*)}{\mathrm{d} \omega^*}$ can be calculated easily.

\subsection{Improving Data Valuation via Gradient Trajectory}\label{method:GIF}
In Theorem~\ref{HIF-complex}, we have to know the exact worst perturbation $\hat{\epsilon}(\omega)$.  However, according to~\citep{foret2020sharpness}, to accelerate the training, the training algorithm for SAM uses the following closed-form estimation of  $\hat{\epsilon}(\omega^*)$ in (\ref{eq:6}) to replace the procedure of finding the model perturbation that results in the largest loss. 
Besides, for the outer-level gradient descent, the second term in (\ref{m1:gradient}) is intentionally dropped during training for simplification, and the final gradient approximation they use is actually $\nabla L_S(w + \hat{\varepsilon}(w))$. See Algorithm \ref{alg:1} for details. 
Therefore, the SAM-HIF we proposed based on the minimization condition represents an idealized case, which sometimes proves insufficiently reliable in practical applications where we cannot find the optimal minimizer. 

\begin{algorithm}
\caption{SAM Algorithm in \citep{foret2020sharpness}  \label{alg:1}}   
\begin{algorithmic}[1]
\REQUIRE Training set $S$, Loss function $L$, Batch size $b$, Step size $\eta > 0$, Neighborhood size $\rho > 0$.  \\
\ENSURE Model trained with SAM  \\
\STATE Initialize weights $w_0$, $t = 0$;  \\
\WHILE{not converged} 
    \STATE Sample batch ${B} = \{(x_1, y_1), \ldots, (x_b, y_b)\}$;  \\
    \STATE Compute gradient $\nabla_w L_{\mathcal{B}}(w_t)$ of batch’s training loss;  \\
    \STATE Compute $\hat{\epsilon}(w)$ per Equation (\ref{eq:6});  \\
    \STATE Compute gradient: $g = \nabla_w L_{{B}}(w_t + \hat{\epsilon}(w_t))$;  \\
    \STATE Update weights: $w_{t+1} = w_t - \eta g$; $t = t + 1$;  \\
\ENDWHILE 
\end{algorithmic}
\end{algorithm}
To bridge the gap, we will next focus on the SAM model trained via the canonical training algorithm in~\citep{foret2020sharpness}, where we have the training trajectory stored in checkpoints. To measure the influence of the $k$-th data in the training dataset, similar to HIF, we firstly define a loss function with the $k$-th term up-weighted by $\delta$ as:
\begin{equation}\label{GIF:loss}
    L_S(\omega;\delta) = L_S(w) + \delta \cdot L_S^k(\omega).
\end{equation}
Consider the model trained based on loss function (\ref{GIF:loss}) using a gradient descent optimizer in~\citep{foret2020sharpness}. At training step $t$, model parameters are updated as
\begin{align*}
 &\omega_{t, \delta} - \omega_{t-1, \delta} \\
=& - \eta_{t-1}\cdot \nabla_{\omega} L_S\left(\omega_{t-1}+\hat{\epsilon}(\omega_{t-1})\right)\\
& - \delta\cdot  \eta_{t-1}\nabla_{\omega}  L_S^k\left(\omega_{t-1,\delta}+\hat{\epsilon}(\omega_{t-1,\delta})\right),
\end{align*}
where $\eta_{t-1}$ is the learning rate at step $t-1$,  and 
$\hat{\epsilon}(\omega_{t-1}, \delta )$ is defined in (\ref{eq:6}). Then, the model trained after $T$ steps becomes 
\begin{equation*}
     \omega_{T, \delta} = \omega_0 - \sum_{t = 0}^{T-1}\eta_{t}\cdot \nabla_{\omega} L_S\left(\omega_{t,\delta}+\hat{\epsilon}(\omega_{t,\delta});\delta\right). 
\end{equation*}
Based on the above discussion, we derive the Gradient-based Influence Function (SAM-GIF) by calculating the derivative of the up-weighted retrained parameter $\omega_{T, \delta}$ for the $k$-th data point with respect to $\delta$.
   \begin{align*}
    &\left. \frac{\mathrm{d}\omega_{T, \delta}}{\mathrm{d}\delta}\right|_{\delta = 0}\\
     =& - \sum_{t = 0}^{T-1}\eta_{t}\cdot \left.  \frac{\mathrm{d} \nabla_{\omega} L_S\left(\omega_{t, \delta}+\hat{\epsilon}(\omega_{t, \delta});\delta\right)}{\mathrm{d} \delta}  \right|_{\delta=0}\\
    &- \sum_{t = 0}^{T-1}\eta_{t}\cdot  {H}_{t,0}\cdot
    \left.\frac{\mathrm{d} (\omega_{t, \delta} + \hat{\epsilon}(\omega_{t, \delta}))}{\mathrm{d} \delta}  \right|_{\delta=0},
\end{align*}
where $H_{t,0}$ is the Hessian matrix of $L_S(\omega_t)$.
To reduce the calculation complexity, we omit the second Hessian term and thus have the following approximation:
 \begin{equation*}
      \text{SAM-GIF}_{\text{GD}}(x_k,y_k) = - \sum_{t = 0}^{T-1}\eta_{t}\cdot \nabla_{\omega}L^k_S(\omega_t + \hat{\epsilon}(\omega_t)).
\end{equation*}
This definition is consistent with~\citep{pruthi2020estimating}. However, compared with gradient descent, Sharpness-Aware Minimization is more often carried out under the framework of SGD optimizer. Under this setting, it is possible that the data $(x_k,y_k)$ may not be used in some updating steps. Therefore, we utilize $B_{k,t}$ to indicate whether $(x_k, y_k)$ is used in the $t$-th gradient descent step. Finally, we have the following result. 
\begin{theorem}\label{GIF}
Assume the model is trained by SGD with the worst perturbation $\hat{\epsilon}(\cdot)$ in each iteration is calculated via (\ref{eq:6}). For the data $(x_k,y_k)$, it  Gradient trajectory-based IF (SAM-GIF) is defined as 
\begin{equation*}
\begin{split}
    \text{SAM-GIF}_{\text{SGD}}(x_k,y_k)
    = \sum_{t=0}^{T-1} \eta_t B_{{k,t}} \nabla_\omega L^k_S(\omega_t + \hat{\epsilon}(\omega_t))
\end{split}
\end{equation*}
Then, the LOO retrained model $\omega_k$ under SAM can be estimated by
\begin{equation*}
   \omega_k\approx \omega^* - \text{SAM-GIF}_{\text{SGD}}(x_k,y_k).
\end{equation*}
\end{theorem}

\subsection{Computation Acceleration and Practical Applications}
\subsubsection{Computation Acceleration}\label{accelerate}
The SAM-HIF outlined in Section \ref{HIF} requires calculations of the inverse Hessian-vector product (iHVP). To enhance the scalability of SAM-HIF, we will introduce one efficient acceleration technique to expedite the computation of iHVP.

\noindent{\bf Neumann Series Approximation Method.}
The calculation in Proposition \ref{HIF-easy} and Theorem \ref{HIF-complex} is expressed as $-H^{-1}\cdot G$, where $H$ denotes the Hessian and $G$ is the gradient. Then, with the help of the Neumann series:
\begin{equation*}
    \begin{split}
        H^{-1} \cdot G &= \left(I - (I-H) \right)^{-1} \cdot G = G + \sum_{j=1}^{+\infty} (I-H)^j \cdot G.
    \end{split}
\end{equation*}
By truncating this series at order $J$, we derive the following approximation:
$$H^{-1} \cdot G \approx G+ (I-H) \cdot G+ \cdots (I-H)^J \cdot G.$$
It is important to clarify that we did not specifically focus on accelerating the inversion of the Hessian matrix; instead, we optimized the iHVP process directly. This approach eliminates the need to store the large Hessian matrix during the computation, which significantly reduces the memory requirements of our method.

\subsubsection{Practical Applications}\label{application}
This section presents several downstream tasks based on the previously derived SAM-HIF and SAM-GIF for different scenarios.

\noindent {\bf Model Editing.}
We can use IFs to update the model under the removal of certain data. Specifically, the model after removing the $k$-th data point can be estimated as $\omega^*- \text{SAM-HIF}(x_k,y_k;\omega^*)$. Additionally, if the training gradient trajectory is available, we can obtain an efficient and accurate estimate as $\omega^*- \text{SAM-GIF}(x_k,y_k;\omega^*)$.

\noindent{\bf Data Evaluation.}
By appropriately selecting the model evaluation function, we can define influence scores (ISs) for individual data points. These scores can then be applied to data selection to enhance SAM performance.
\begin{definition}
{\bf (Evaluation Function)}
Given a validation  dataset defined as $D_{val} = \{(x_t,y_t)\}_{t=1}^n$. the SAM learned parameter $\omega^*$ performance on the test task is defined as $\sum_{(x,y)\in D_{val}}\ell(x,y;\omega^*)$.
\end{definition}
Based on this metric, we can propose an evaluation method for the data influence. 
\begin{theorem}\label{IS}
Given a validation dataset defined as $D_{val} = \{(x_t,y_t)\}_{t=1}^n$.
Denote the SAM-retrained model after the removal of $(x_k,y_k)$ as $\omega^*_{-k}$, then 
\begin{equation*}
\begin{split}
       &\sum_{(x,y)\in D_{val}}\ell(x,y;\omega^*) - \sum_{(x,y)\in D_{val}}\ell(x,y;\omega^*_{-k}) \\
       \approx &\sum_{(x,y)\in D_{val}} \nabla \ell(x,y;\omega^*) \cdot \text{IF}(x,y)\triangleq \text{IS}(x,y), 
\end{split}
\end{equation*}
where IF can be  SAM-HIF or SAM-GIF. We define the right hand of the above equation as the Influence Score (IS). 
\end{theorem}
A positive IS indicates that removing the data point will deteriorate the model's performance on the test dataset. Thus, this data point is valuable for model performance. Assigning an IS to each training data point allows us to identify useful and harmful data for the model’s performance. Detailed theoretical derivations are provided in the Appendix.

\section{Experiments}\label{sec:exp}
In this section, we demonstrate our main experimental results on utility, efficiency, effectiveness, and the abilities to identify harmful data and enhance interpretability.  Details and additional results are in Appendix due to space limit.

\subsection{Experimental Settings}
\noindent{\bf Dataset.} To evaluate the performance of our algorithm, we conducted experiments using four datasets: CIFAR10, CIFAR100~\cite{alex2009learning}, MiniImageNet~\cite{5206848}, MNIST~\cite{lecun1998gradient}, and HAM10000~\cite{tschandl2018ham10000}. %The CIFAR-10 and CIFAR-100 datasets consist of color images categorized into 10 and 100 classes, respectively, and are widely used for small-scale image classification tasks. 
MiniImageNet, a subset of the larger ImageNet dataset, is commonly employed in few-shot learning research due to its diverse class representations and moderate resolution. %MNIST, a classic handwritten digit recognition dataset, contains images of digits from 0 to 9, totaling 10 classes.

\noindent{\bf Baselines.} We employ retraining as the baseline method since the data valuation problem in SAM has never been investigated, and there are no other baselines. \textit{Retrain}: We retrain the model via SAM after removing the data from the training set. 
In scenarios where the model training trajectory cannot be obtained, we employ SAM-HIF (fast) and SAM-HIF. \textit{SAM-HIF (fast)}: SAM-HIF (fast) is a direct implementation of our previous Theorem \ref{HIF-easy}. \textit{SAM-HIF}: SAM-HIF is a direct implementation of our previous Theorem \ref{HIF-complex}. 
    
In scenarios where the model training trajectory is available, we employ SAM-GIF. \textit{SAM-GIF}: SAM-GIF is a direct implementation of our previous Theorem \ref{GIF}. For the methods mentioned above, we provide corresponding algorithms in the Appendix. SAM-HIF (fast) and SAM-HIF are accelerated by the Neumann Series approximation method discussed in Section \ref{accelerate}. We provide implementation algorithms for SAM-HIF (fast), SAM-HIF and SAM-GIF in the Appendix.

\noindent{\bf Evaluation Metric.} We used two primary evaluation metrics to assess our models: accuracy and runtime (RT). Accuracy evaluates the model’s performance by measuring the proportion of correctly classified instances out of the total instances. Runtime(RT), measured in seconds, assesses the time required for each method to update the model.

\noindent{\bf Implementation Details.} We conducted experiments using the Nvidia RTX 4090-24G GPU. For utility evaluation, we randomly selected samples at different proportions from four datasets. For valuable (or harmful) samples, we removed 0-10 \% of the data identified as valuable (or harmful) by the algorithm. Additionally, the removal process was repeated five times using different random seeds to obtain experimental results.

\subsection{Evaluation of Utility and Editing Efficiency}
We first demonstrate the results on accuracy and time consumption of three algorithms, SAM-HIF (fast), SAM-HIF, and SAM-GIF, against the retrain method. We selected the WideResNets architecture as the backbone network for the classification task. Our experimental results are presented in Table~\ref{tab:results:main}. It is evident that our three proposed algorithms significantly improve computational efficiency without sacrificing accuracy. On the CIFAR-10 dataset, the time cost of retraining reached 3516.74, while our methods notably enhance computational efficiency, with SAM-HIF (fast), SAM-HIF, and SAM-GIF reducing the time to 11.0698, 41.4960, and 4.894 seconds, respectively. The accuracy differences between SAM-HIF (fast), SAM-HIF, and SAM-GIF compared to retrain are 0.0177, 0.0057, and 0.0003, respectively. These findings suggest that SAM-HIF (fast), SAM-HIF, and SAM-GIF can save substantial computational time required for retraining while achieving comparable accuracy. 
\begin{table*}[ht]
\centering
\caption{Performance comparison on CIFAR-10, CIFAR-100, and MINI-Imagenet.}
\label{tab:results:main}
\resizebox{\linewidth}{!}{
\begin{tabular}{lcccccc}
\toprule
    \multirow{3}{*}{\textbf{Method}} & \multicolumn{2}{c}{\textbf{Cifar10}} & \multicolumn{2}{c}{\textbf{Cifar100}} & \multicolumn{2}{c}{\textbf{MINI-Imagenet}} \\
    \cmidrule(r){2-3} \cmidrule(r){4-5} \cmidrule(r){6-7} 
    & \textbf{Accuracy} & \textbf{RT (second)} & \textbf{Accuracy} & \textbf{RT (second)} & \textbf{Accuracy} & \textbf{RT (second)} \\
\midrule
    Retrain & 0.9500$\pm$0.0061 & 3516.74$\pm$8.45 & 0.7890$\pm$0.0311 & 3244.24$\pm$2.18 & 0.6835$\pm$0.0460 & 682.56$\pm$7.91 \\
    SAM-HIF(Fast) & 0.9323$\pm$0.0110 &  11.0698$\pm$3.41 & 0.7208$\pm$0.0142 &  13.0698$\pm$4.16 & 0.6478$\pm$0.0125 &  11.3218$\pm$2.43 \\
    SAM-HIF & 0.9443$\pm$0.0121 & 41.4960$\pm$3.21 & 0.7213$\pm$0.0239 & 42.412$\pm$3.24 & 0.6516$\pm$0.0213 & 39.212$\pm$3.09 \\
    SAM-GIF & 0.9497$\pm$0.0142 & 4.8942$\pm$1.42 & 0.7227$\pm$0.0469 & 6.8942$\pm$2.31 & 0.6446$\pm$0.0122 & 5.8942$\pm$1.122 \\
\bottomrule
\end{tabular}}
\vspace{-7pt}
\end{table*}

Furthermore, we observe that by avoiding the computation of the Hessian matrix, SAM-GIF requires less time than both SAM-HIF (Fast) and SAM-HIF. Besides, SAM-GIF not only surpasses SAM-HIF and SAM-HIF (Fast) in speed but also achieves accuracy comparable to retraining. This corroborates our previous discussion: compared to the SAM model, gradient trajectories can more accurately reflect the training of the SAM model, leading to a more precise estimation result. By avoiding the computation of the Hessian matrix, SAM-GIF demonstrates better time efficiency than other methods while also showing improved accuracy on CIFAR-10. These findings indicate that utilizing gradient information effectively enhances both training efficiency and predictive accuracy. Moreover, we can see SAM-HIF achieves higher accuracy than SAM-HIF (Fast) by incorporating perturbation-related components into the Hessian matrix; however, this comes at the cost of a 2-3$\times$ increase in runtime. Mathematically, as we mentioned in Theorem \ref{HIF-complex}, SAM-HIF further considers the change of perturbations due to the data removal. This can make it evaluate data influence more accurately but with additional computation time. Please refer to section \ref{app:exp:efficiency} for more results.

We also compare with several baselines. We choose two algorithms: TARK~\cite{park2023trak} and IF-EKFAC. The results are listed in Table \ref{tab:results_short}. We tested it on the Cifar-10 dataset. We first trained the Wide-Resnet model for 60 epochs and obtained an accuracy of 0.8132. Then, we compute the scores for all the training data and retrain by removing the bottom 10\% of the data. For SAM-GIF, we also use our algorithm to estimate the parameters.

\begin{table*}[ht]
\centering
\caption{Performance comparison of different methods on Cifar10 dataset.}
\label{tab:results_short}
\resizebox{0.5\linewidth}{!}{
\begin{tabular}{lcc}
\toprule
    \multirow{2}{*}{\textbf{Method}} & \multicolumn{2}{c}{\textbf{Cifar10}} \\
    \cmidrule(r){2-3}
    & \textbf{Accuracy} & \textbf{RT (second)} \\
\midrule
    TARK & 0.8321$\pm$0.0061 & 3314.23$\pm$10 \\
    IF-EKFAC & 0.8323$\pm$0.0110 &   3338.94$\pm$10 \\
    SAM-GIF(Retrain) &  0.8423$\pm$0.0142 & 3231.49$\pm$10 \\
    SAM-GIF(GIF) & 0.8401$\pm$0.0142 & 5.1442$\pm$1.42 \\
\bottomrule
\end{tabular}}
\vspace{-2pt}
\end{table*}

Furthermore, we perform experiments on Resnet-50~\cite{he2016deep} and the results are shown in Table \ref{tab:50-results}.

\begin{table*}[ht]
\centering
\caption{Performance comparison on \textit{Resnet50}.}
% \vspace{-10pt}
\label{tab:50-results}
\resizebox{0.8\textwidth}{!}{ % 控制整体缩放
\begin{tabular}{lcccc}
\toprule
    \multirow{3}{*}{\textbf{Method}} & \multicolumn{2}{c}{\textbf{CUB}} & \multicolumn{2}{c}{\textbf{FOOD-101}} \\
    \cmidrule(r){2-3} \cmidrule(r){4-5}
    & \textbf{Accuracy} & \textbf{RT (s)} & \textbf{Accuracy} & \textbf{RT (s)} \\
\midrule
    Retrain & 0.6336$\pm$0.005 & 2197.33$\pm$12.47 & 0.7436$\pm$0.031 & 15900.24$\pm$25.32 \\
    SAM-HIF & 0.6213$\pm$0.005 & 37.2590$\pm$3.78 & 0.7022$\pm$0.024 & 41.321$\pm$4.12 \\
     SAM-HIF(Fast) & 0.6178$\pm$0.003 & 25.3214$\pm$5.16 & 0.6918$\pm$0.014 & 26.0698$\pm$3.21 \\
    SAM-GIF & 0.6244$\pm$0.004 & 12.6789$\pm$2.33 & 0.7189$\pm$0.047 & 13.234$\pm$3.23 \\
\bottomrule
\end{tabular}}
\label{tab:dataset}
\end{table*}

\begin{figure}[htbp]
    \centering
    \begin{subfigure}{0.37\linewidth}
        \centering
\includegraphics[width=\linewidth]{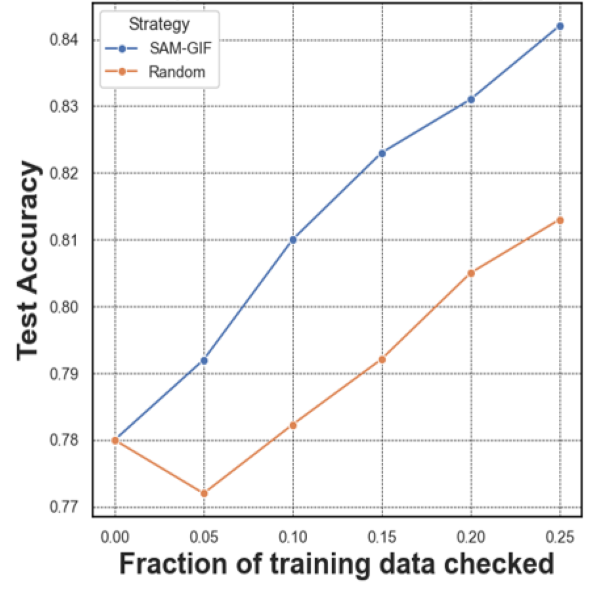}
        \caption{Test Accuracy}
        \label{fig:hdr:ad}
    \end{subfigure}%
    \begin{subfigure}{0.37\linewidth}
        \centering
        \includegraphics[width=\linewidth]{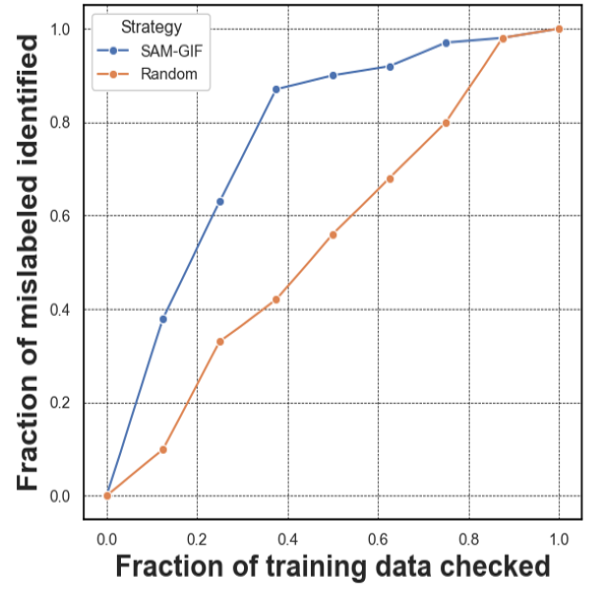}
        \caption{Harmful Data Removal}
        \label{fig:accuracy:ad}
    \end{subfigure}
    \caption{Harmful data removal experiment. IS means using the influence score to determine which sample to remove. Random refers to randomly removing tasks}
    \label{fig:3}
\vspace{-12pt}
\end{figure}
\subsection{Evaluation of Effectiveness for Data Attribution}
Here, we aim to show that our proposed influence functions can be used to attribute data.  To achieve this, we traverse the training samples by calculating the IS value for each sample to select the most valuable data. We then remove these samples and update the model through retraining or using the SAM-HIF (Fast), SAM-HIF, and SAM-GIF, and analyzing the changes in accuracy relative to the accuracy before removal. In detail, we first calculate the influence scores of all samples and select the most valuable ones by ranking them. Then, the top k most valuable data points are removed, where k ranges from 2\% to 10\% of the data size. Next, the model is retrained multiple times using different random seeds or updated using our different methods. The test set accuracy is then calculated after updating the parameters. Additionally, we randomly delete the same proportion of samples as a control experiment.

Our experimental results, as shown in the Figure \ref{fig:datasets}, indicate that after removing the most valuable training data, the accuracy of the updated model decreases across different datasets. For instance, for the CIFAR-10 dataset, when the removal ratios are 0.02, 0.05, 0.07, and 0.10, the accuracy scores for retraining and the SAM-HIF (Fast), SAM-HIF, and SAM-GIF algorithms gradually decrease. At the 0.10 removal ratio, the accuracy for the four methods drops to 0.8432, 0.8212, 0.8412, and 0.8311, respectively. The accuracy of different approaches is similar to that of retraining (with a maximum difference of 0.022), which demonstrates that our algorithm significantly reduces computational time while maintaining accuracy. In contrast, random deletion shows more significant variability, and while an increase in the deletion ratio may lead to an increase in accuracy, it is less stable.

\subsection{Results of Identifying Harmful Data}
\begin{figure}
    \centering
    \begin{subfigure}[b]{0.45\linewidth}
        \centering
        \includegraphics[width=\linewidth]{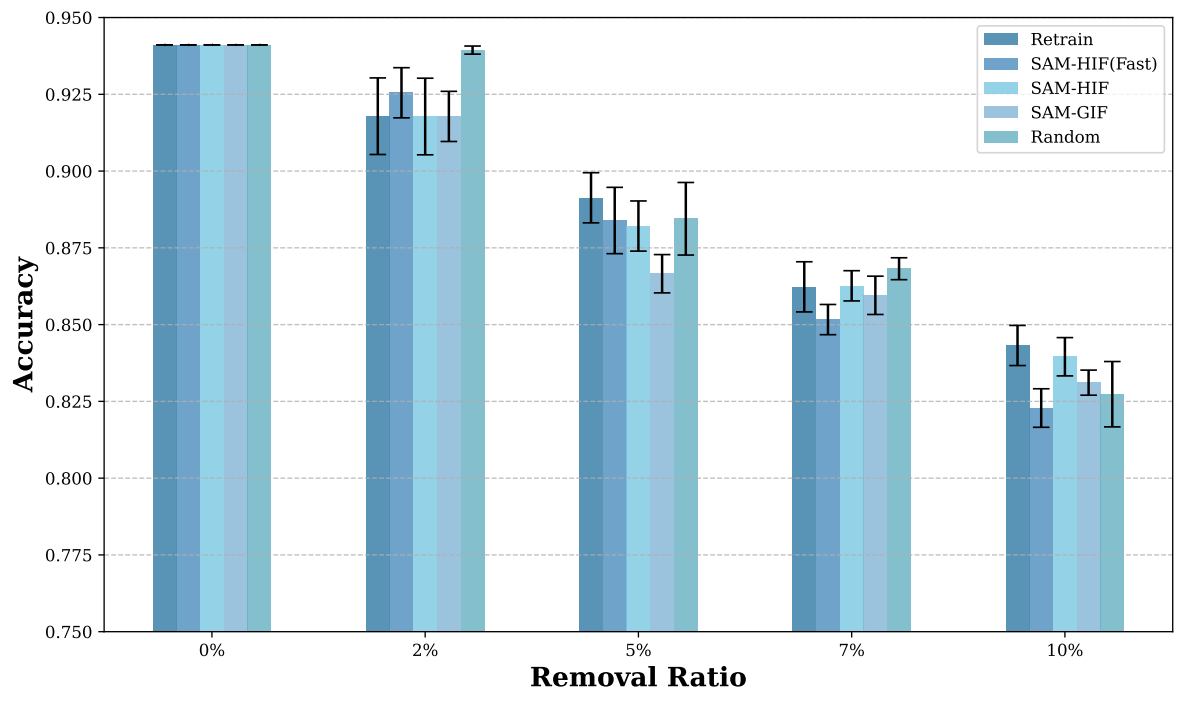}
        \caption{Cifar10}
        \label{fig:cifar100}
    \end{subfigure}
    \hfill
    \begin{subfigure}[b]{0.45\linewidth}
        \centering
        \includegraphics[width=\linewidth]{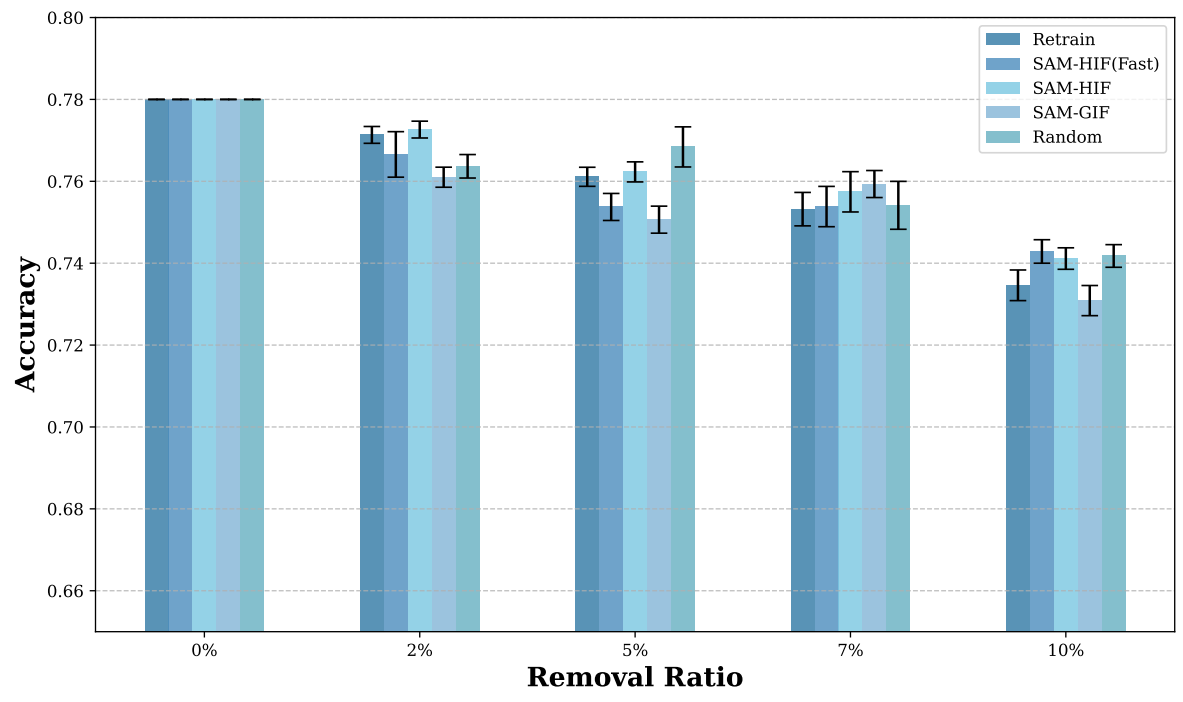}
        \caption{Cifar100}
        \label{fig:cifar10}
    \end{subfigure}

    \vspace{1em} % Adjust vertical space between rows

    \begin{subfigure}[b]{0.45\linewidth}
        \centering
        \includegraphics[width=\linewidth]{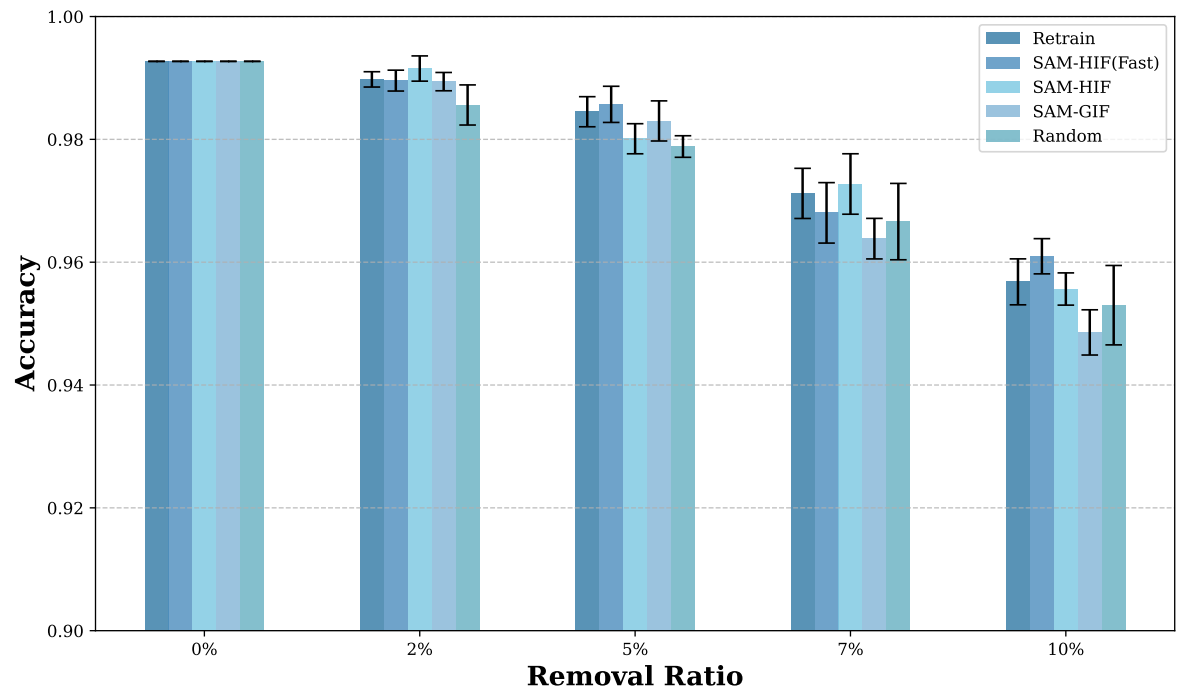}
        \caption{Mnist}
        \label{fig:mnist}
    \end{subfigure}
    \hfill
    \begin{subfigure}[b]{0.45\linewidth}
        \centering
        \includegraphics[width=\linewidth]{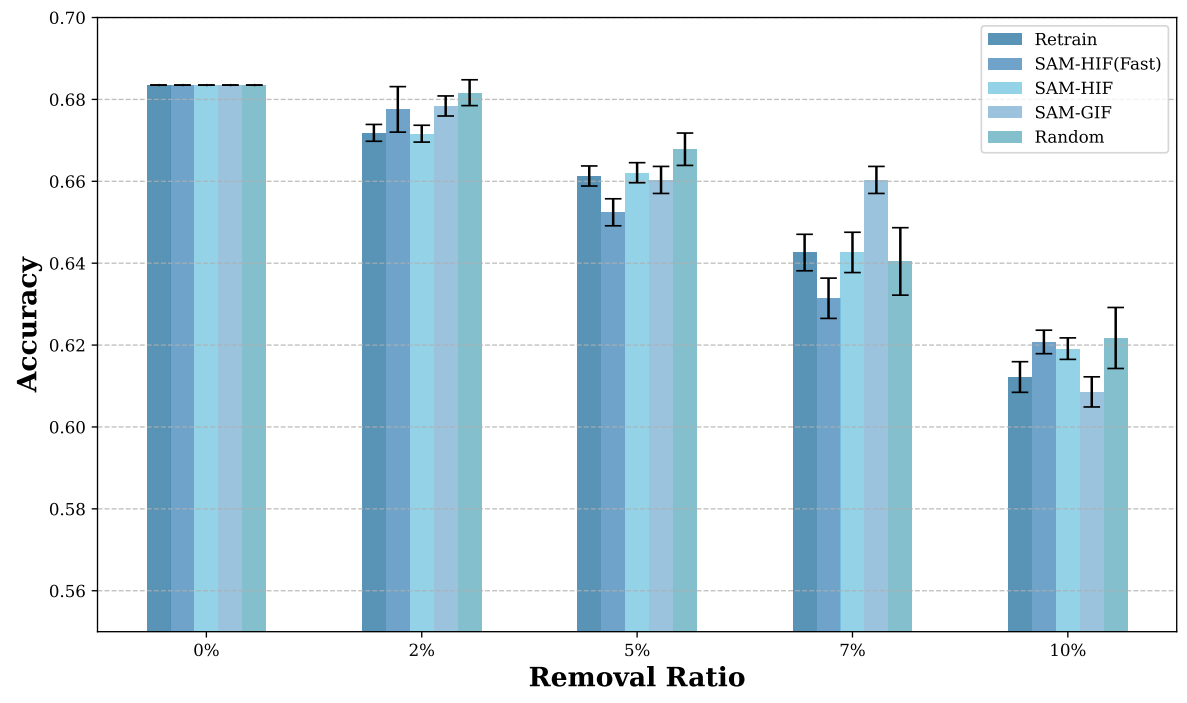}
        \caption{Miniimagenet}
        \label{fig:miniimagenet}
    \end{subfigure}
    \vspace{-5pt}
    \caption{Accuracy performance on four datasets after removing varying proportions of helpful data using task-related IS}
    \label{fig:datasets}
\vspace{-12pt}
\end{figure}
 To demonstrate the ability of our algorithm to identify harmful data for a specific task, we introduce noise into the CIFAR-10 dataset and conduct experiments on the noisy dataset. Specifically, we randomly select 10\% of the CIFAR-10 data samples and invert their labels, then use the noisy dataset to train the model using the SAM method. The training setup remains consistent with previous configurations. We compute the IS of the samples using SAM-GIF, and then identify the data as harmful if its IS is low.  We then compare the model accuracy after retraining with different proportions of harmful data points randomly removed. Figure \ref{fig:3}a shows the change in model accuracy after removing varying proportions of harmful data, indicating that as the proportion of harmful data removed increases, the model's accuracy on the test set also improves. When the removal rate reaches 0.25, the accuracy increases by approximately 0.06. For random deletion, the accuracy exhibits considerable instability; initially, the accuracy decreases around a removal rate of 0.05, but it gradually increases as the removal rate increases. Figure \ref{fig:3}b illustrates the proportion of noisy data detected after removing training samples. Compared to random deletion, our algorithm demonstrates higher efficiency in detecting noisy data. At a detection rate of approximately 0.4, we have identified over 90\% of the noisy data. The experimental results for the remaining datasets can be found in the Appendix. 

 We also compare the performance of our method with TARK~\cite{park2023trak}, Influence Function (IF), Influence Function-EKFAC (IF-EKFAC) in the harmful data recogonization task. 

\begin{figure}
    \centering
\includegraphics[width=0.9\linewidth]{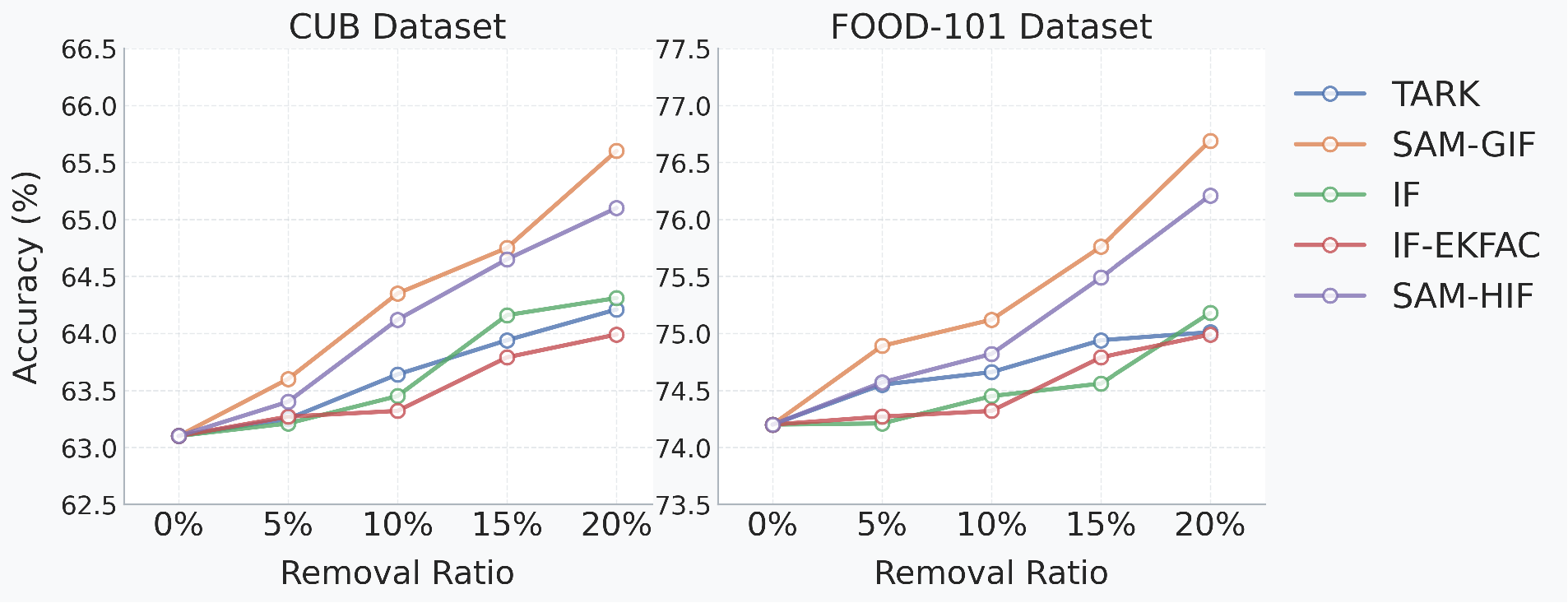}
    \caption{Harmful data removal results compared with baselines.}
    \label{fig:baselin}
   \vspace{-7pt}
\end{figure}

\subsection{Results on Interpretability}
\begin{figure}
    \centering
    \includegraphics[width=0.95\linewidth]{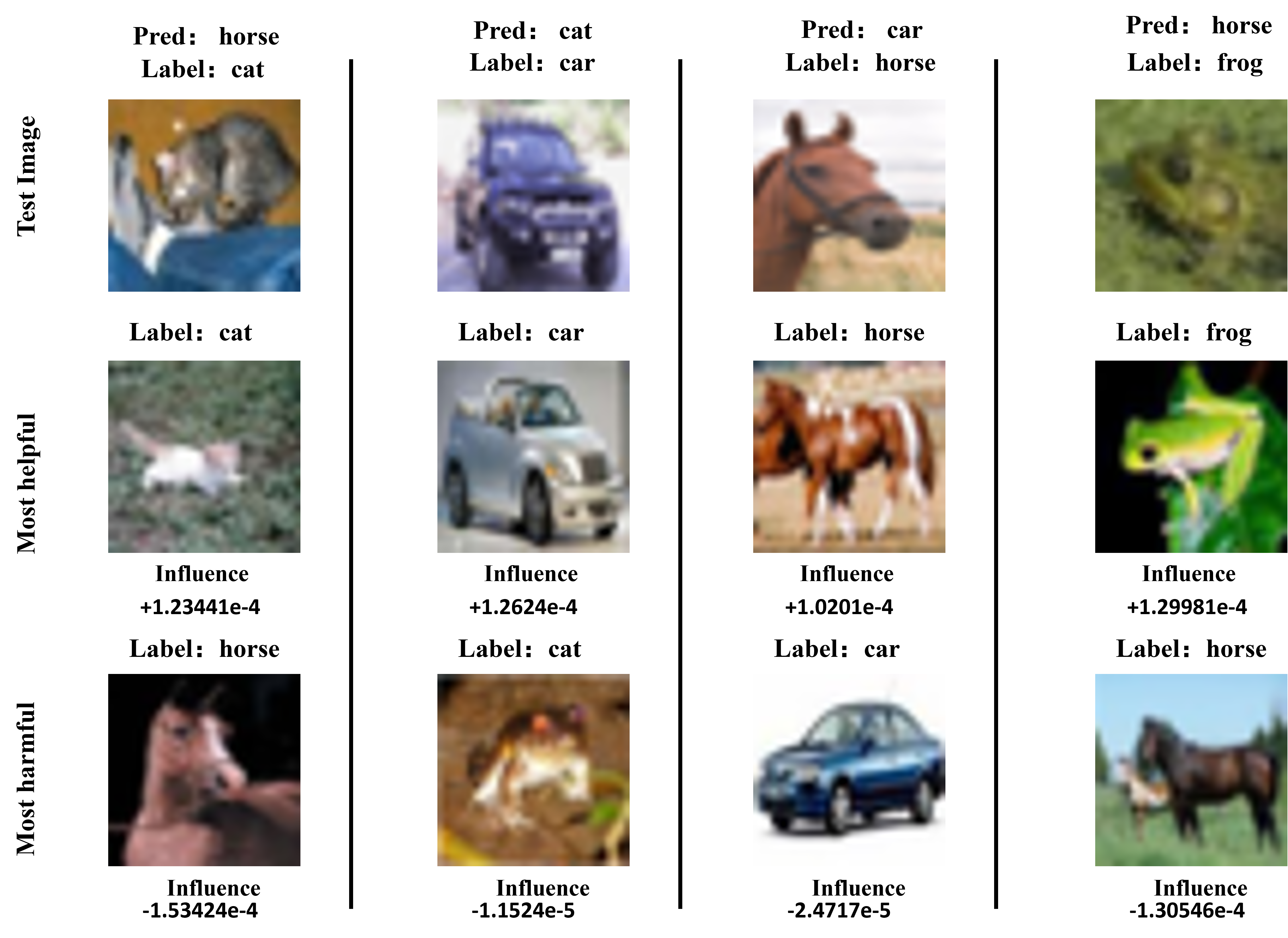}
    \caption{The most helpful and harmful training data tracked by misclassified data}
    \label{fig:enter-label}
\vspace{-15pt}
\end{figure}
We apply SAM-GIF to identify the training samples that are most relevant to specific prediction errors in the test samples. In this process, we select samples from the test data where the model misclassifies. By calculating the IS of the samples, we identify and visualize harmful data with a negative IS that leads to erroneous predictions, namely. Figure~\ref{fig:enter-label} presents the visualization of this error prediction tracing process for the CIFAR-10 dataset (more visualization results are in Appendix). The first row displays examples of misclassified test samples, the second row shows the most influential training data for classifying this sample, and the third row shows the most harmful training data for this classification. We can trace the outcomes of the error prediction process through the visualization results.

\subsection{Ablation Study}
We conducted ablation experiments on above three methods, SAM-HIF (fast), SAM-HIF and SAM-GIF. We randomly removed 1\%-8\% of the training samples from CIFAR-10 and CIFAR-100 datasets, and evaluated the model parameters using SAM-HIF(fast) and SAM-HIF, with retrain serving as the ground truth. The results, shown in Figure \ref{fig:4}, clearly demonstrate that SAM-HIF consistently outperforms SAM-HIF(Fast) during the process of data removal. Also, we can see both methods consistently have slight differences with the ground truth. Results for SAM-GIF are included in the Appendix.
\begin{figure}[htbp]
    \centering
    \begin{subfigure}{0.37\linewidth}
        \centering
        \vspace{-5pt}
        \includegraphics[width=\linewidth]{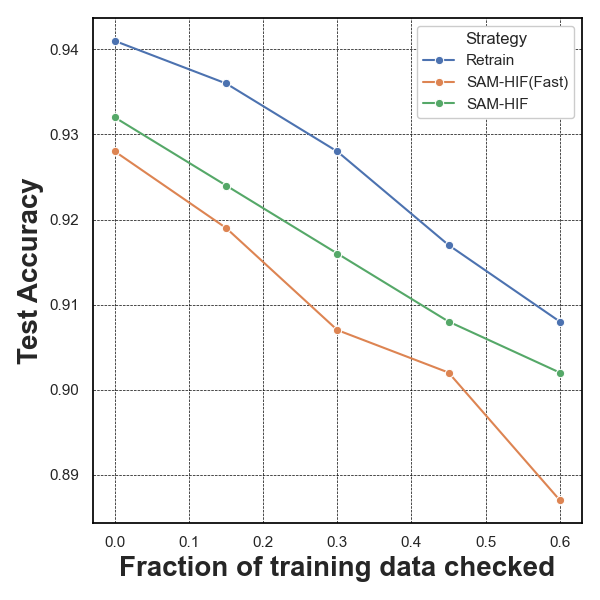}
        \caption{CIFAR-10}
        \label{fig:hdr:abl}
    \end{subfigure}%
    \begin{subfigure}{0.37\linewidth}
        \centering
        \includegraphics[width=\linewidth]{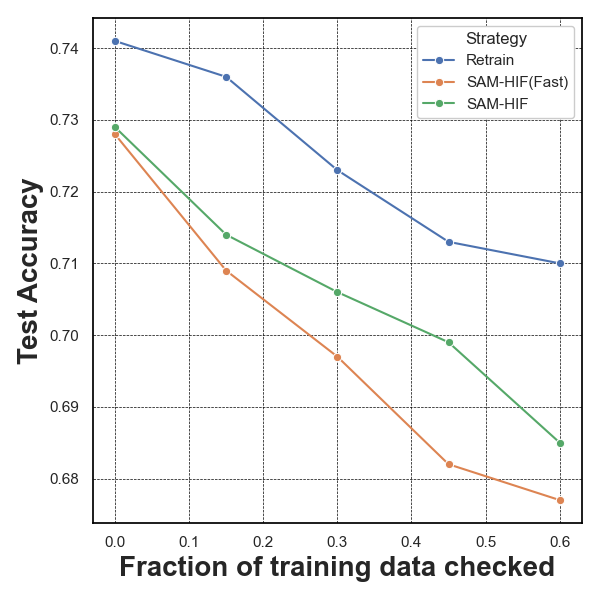}
        \caption{CIFAR-100}
        \label{fig:accuracy:abl}
    \end{subfigure}
    \caption{Ablation studies on CIFAR-10 and CIFAR-100.}
    \label{fig:4}
\vspace{-15pt}
\end{figure}

\section{Conclusion}
In conclusion, we addressed data attribution challenges in the SAM framework with two novel methods: SAM-HIF and SAM-GIF. SAM-HIF employs a comprehensive closed-form data influence estimation. SAM-GIF utilizes gradient trajectory information for efficient, scalable evaluation. Experiments on four datasets demonstrated the effectiveness and scalability of these methods. Our approach aids in mislabeled data detection, model editing, and interpretability in SAM-trained models.

\bibliographystyle{plain}
\bibliography{main}

%%%%%%%%%%%%%%%%%%%%%%%%%%%%%%%%%%%%%%%%%%%%%%%%%%%%%%%%%%%%
\clearpage
\newpage
\appendix

\section{Omitted Proofs}

\subsection{Evaluating Data Attribution in SAM via Hessian-based IF}

\begin{theorem}\label{app:HIF-easy}
    Consider the data $(x_k, y_k)$ along with a SAM-trained model $\omega^*$, $\hat{\epsilon}(\omega^*)$ represents the model weight perturbation that results in the largest loss.
Then the corresponding SAM-IF is defined as:
\begin{align*}
    \text{SAM-IF}(x_k, y_k) &= -H_{\omega}^{-1} \cdot\nabla L_S^k(\omega^*+ \hat{\epsilon}(\omega^*))
\end{align*}
where $H_{\omega}$ is defined as $H_{\omega} = \nabla^2 L_S\left(\omega^*+\hat{\epsilon}(\omega^*)\right)+ \lambda\cdot I$. $\omega_{k}$ can be approximated by 
\begin{equation*}
    \omega_{k}  \approx  \omega^*-\text{SAM-IF}(x_k, y_k).
\end{equation*}
\end{theorem}

\begin{proof}
Recalling the definition $\hat{\epsilon}(\omega) = \argmax_{||\epsilon||_p\leq\rho} L_S(\omega+\epsilon)$ and
\begin{equation}\label{app:optimal}
    \begin{split}
        \omega^* &= \argmin_{\omega} L_S(\omega) + 
        \frac{\lambda}{2} \cdot ||\omega||^2_2\\
        &=\argmin_{\omega} L_S(\omega +  \hat{\epsilon}(\omega)) + \frac{\lambda}{2}\cdot ||\omega||^2_2\\
        &= \argmin_{\omega} \sum_{i=1}^n\ell(x_i,y_i;\omega+ \hat{\epsilon}(\omega)) + \frac{\lambda}{2}\cdot ||\omega||^2_2\\
        &= \argmin_{\omega} \sum_{i=1}^n L_S^i(\omega +  \hat{\epsilon}(\omega))  + \frac{\lambda}{2}\cdot ||\omega||^2_2
    \end{split}
\end{equation}    
Now we consider removing the $k$-th data $(x_k, y_k)$. The retrained parameter after the removal is denoted as $\omega_k$:
\begin{equation*}
    \omega_k = \argmin_{\omega} \sum_{i=1, i\neq k}^n L_S^i(\omega +  \hat{\epsilon}(\omega))  + \frac{\lambda}{2}\cdot ||\omega||^2_2
\end{equation*}
To estimate this $\omega_k$, we can firstly up-weigh the loss term $L_S^k(\omega +  \hat{\epsilon}(\omega))$ by $\delta$ in the original loss function defined in (\ref{app:optimal}), then a series of optimization problem can be defined as:
    \begin{align}\label{app:optimal_delta}
        \omega_{\delta} = \argmin_{\omega}L_S(\omega +  \hat{\epsilon}_{\delta}(\omega)) + \delta \cdot L_S^k(\omega+ \hat{\epsilon}_{\delta}(\omega)) + \frac{\lambda}{2}\cdot ||\omega||^2_2
    \end{align}
Noting here $\hat{\epsilon}_{\delta}(\omega)$ will change according to $\delta$ and is defined as
    \begin{equation*}
        \hat{\epsilon}_{\delta}(\omega) = \argmax_{||\epsilon||_p\leq\rho}  \left[ L_S(\omega+ {\epsilon})+  \delta \cdot L_S^k(\omega+ {\epsilon}) \right]
    \end{equation*}
From the minimizing condition in Equation (\ref{app:optimal}) and (\ref{app:optimal_delta}) and Lemma \ref{danskin}, we have
    \begin{equation}\label{app:optimal_condition}
        \begin{split}
\nabla L_S(\omega^*+\hat{\epsilon}(\omega^*)) = 0,\indent\nabla L_S(\omega_{\delta}+\hat{\epsilon}_{\delta}(\omega_{\delta})) +  \delta\cdot\nabla L^k_S(\omega_{\delta}+\hat{\epsilon}_{\delta}(\omega_{\delta}))= 0.
        \end{split}
    \end{equation}
To estimate $\omega_{\delta}- \omega^*$, we perform a Taylor expand at $\omega^*$ for the first equation in (\ref{app:optimal_condition}).:
    \begin{equation}\label{app:taylor_expand}
        \begin{split}
            0  &=   \nabla L_S\left(\omega^*+\hat{\epsilon}(\omega^*)\right) + \delta \cdot \nabla L^k_S\left(\omega^*+\hat{\epsilon}(\omega^*)\right) + \nabla^2 L_S\left(\omega^*+\hat{\epsilon}(\omega^*)\right)\cdot \left(\omega_{\delta} - \omega^*\right)\\
       & +\nabla^2 L_S\left(\omega^*+\hat{\epsilon}(\omega^*)\right) \cdot \left(\hat{{\epsilon}}_{\delta}\left(\omega_{\delta}\right) - \hat{\epsilon}(\omega^*)\right) 
        \end{split}
    \end{equation}
    In the above expansion, the first term equals $0$ from (\ref{app:optimal_condition}). And to make thing easier, we can neglect the last term and deduce a simple influence function for SAM.
    \begin{align*}
        \text{SAM-IF}(x_k, y_k) &= \left.\frac{\mathrm{d}\omega_{\delta}}{\mathrm{d}\delta}\right|_{\delta = 0}= -H_{\omega}^{-1} \cdot\nabla L_S^k(\omega^*+ \hat{\epsilon}(\omega^*)),
    \end{align*}
    where $H_{\omega}$ is defined as $H_{\omega} = \nabla^2 L_S\left(\omega^*+\hat{\epsilon}(\omega^*)\right)+ \lambda\cdot I$. 
    
When $\delta = -1$, $\omega_{\delta}$ becomes $\omega_{k}$. Then $\omega_{k}$ can be approximated by a first-order Taylor expansion:
\begin{equation*}
    \omega_{k}  \approx  \omega^*-\text{SAM-IF}(x_k, y_k).
\end{equation*}
\end{proof}

\begin{lemma}\label{lemma2}
    The term 
   $\hat{{\epsilon}}_{k, \delta}\left(\omega_{k, \delta}\right) - \hat{\epsilon}(\omega^*)$ can be approximated by
   \begin{equation}
   \begin{split}
       &\hat{\epsilon}_{k,\delta}(\omega_{k,\delta}) - \hat{\epsilon}(\omega^*)\\
       =& \hat{\epsilon}_{k, \delta}(\omega^*) - \hat{\epsilon}(\omega^*)  + \left. \frac{\mathrm{d} \hat{\epsilon}_{k, \delta}(\omega)}{\mathrm{d} \omega} \right|_{\omega = \omega^*} \cdot (\omega_{k,\delta} - \omega^*).
   \end{split}
   \end{equation}
   \end{lemma}

\begin{proof}    
Now, we begin to estimate $\hat{\epsilon}_{k,\delta}(\omega_{k,\delta}) - \hat{\epsilon}(\omega^*)$. Recalling the definitions as following:
  \begin{align*}
    \hat{{\epsilon}}_{k, \delta}\left(\omega_{k, \delta}\right)  &=\argmax_{||\epsilon||_p\leq\rho} \left[ L_{S}(\omega_{k, \delta}+ {\epsilon}) + \delta\cdot L_S^k(\omega_{k, \delta}+ {\epsilon}) \right],\\
          \hat{\epsilon}(\omega^*) &= \argmax_{||\epsilon||_p\leq\rho}.
          L_S(\omega^* + {\epsilon})
  \end{align*}
  We will use $\hat{\epsilon}_{k, \delta}(\omega^*)$ as an intermediate variable to simplify the approximation, which is defined as
  \begin{equation*}
    \hat{\epsilon}_{k, \delta}(\omega^*) = \argmin_{||\epsilon||_p\leq\rho}  \left[ L_{S}(\omega^*+ {\epsilon}) + \delta\cdot L_S^k(\omega^*+ {\epsilon}) \right]
  \end{equation*}
  Firstly, we perform a Taylor expansion for $\hat{{\epsilon}}_{k, \delta}\left(\omega_{k, \delta}\right)$ at $\omega^*$ as follows:
  \begin{equation*}
    \hat{{\epsilon}}_{k, \delta}\left(\omega_{k, \delta}\right) \approx \hat{{\epsilon}}_{k, \delta}\left(\omega^*\right)   + \left. \frac{\mathrm{d} \hat{{\epsilon}}_{k, \delta}\left(\omega_{k, \delta}\right)}{\mathrm{d} \omega} \right|_{\omega = \omega^*} \cdot(\omega_{k,\delta} - \omega^*)
  \end{equation*}
  Then the objective to estimate becomes
\begin{align*}\label{lemma2}
    &\hat{\epsilon}_{k,\delta}(\omega_{k,\delta}) - \hat{\epsilon}(\omega^*)\\
=&  \hat{\epsilon}_{k,\delta}(\omega_{k,\delta}) - \hat{\epsilon}_{k,\delta}(\omega^*) + \hat{\epsilon}_{k,\delta}(\omega^*)  - \hat{\epsilon}(\omega^*)\\
=& \hat{\epsilon}_{k,\delta}(\omega^*)  - \hat{\epsilon}(\omega^*) + \left. \frac{\mathrm{d} \hat{{\epsilon}}_{k, \delta}\left(\omega_{k, \delta}\right)}{\mathrm{d} \omega} \right|_{\omega = \omega^*} \cdot(\omega_{k,\delta} - \omega^*)
\end{align*}
\end{proof}

\begin{theorem}\label{app:HIF-complex}
Consider the $k$-th data $(x_k, y_k)$ along with a SAM-trained model $\omega^*$. Define the Hessian-based influence function for SAM(SAM-HIF) as:
\begin{align*}
    &\text{SAM-HIF}(x_k, y_k)\\
    =& -\left(H_{\omega} +H_{\omega}\cdot  \frac{\mathrm{d} \hat{\epsilon}(\omega^*)}{\mathrm{d} \omega }\right)^{-1} \cdot\nabla L^k_S\left(\omega^*+\hat{\epsilon}(\omega^*)\right).
\end{align*}
Then $ \omega_{k}$ can be approximated by:
\begin{equation*}
    \omega_{k} \approx \omega^* -\text{SAM-HIF}(x_k, y_k).
\end{equation*}
\end{theorem}

   \begin{proof}
Based on the expansion in Equation (\ref{app:taylor_expand}) in Theorem \ref{app:HIF-easy}, we can further estimate the term $\hat{{\epsilon}}_{k, \delta}\left(\omega_{k, \delta}\right) - \hat{\epsilon}(\omega^*)$ by Lemma \ref{app:epsilon_approx}:
\begin{equation}
    \begin{split}
        0  &=   \nabla L_S\left(\omega^*+\hat{\epsilon}(\omega^*)\right) + \delta \cdot \nabla L^k_S\left(\omega^*+\hat{\epsilon}(\omega^*)\right) + \nabla^2 L_S\left(\omega^*+\hat{\epsilon}(\omega^*)\right)\cdot \left(\omega_{\delta} - \omega^*\right)\\
   &+ \nabla^2 L_S\left(\omega^*+\hat{\epsilon}(\omega^*)\right)\cdot \left(\hat{\epsilon}_{\delta}(\omega^*)  - \hat{\epsilon}(\omega^*) + \left. \frac{\mathrm{d} \hat{{\epsilon}}_{\delta}\left(\omega_{\delta}\right)}{\mathrm{d} \omega} \right|_{\omega = \omega^*} \cdot(\omega_{\delta} - \omega^*) \right)
    \end{split}
\end{equation}
The first term equals $0$ from (\ref{app:optimal_condition}). Then we have
\begin{equation}
    \begin{split}
        0  &=  \delta \cdot \nabla L^k_S\left(\omega^*+\hat{\epsilon}(\omega^*)\right) + \left(H_{\omega}+ H_{\omega} \cdot\frac{\mathrm{d} \hat{{\epsilon}}_{\delta}\left(\omega_{\delta}\right)}{\mathrm{d} \omega}\right)\cdot \left(\omega_{\delta} - \omega^*\right)+ H_{\omega}\cdot \left(\hat{\epsilon}_{\delta}(\omega^*)  - \hat{\epsilon}(\omega^*) \right).
    \end{split}
\end{equation}
Then 
\begin{align*}
\omega_{\delta} - \omega^*  = &- \left(H_{\omega}+ H_{\omega} \cdot\frac{\mathrm{d} \hat{{\epsilon}}_{\delta}\left(\omega_{\delta}\right)}{\mathrm{d} \omega}\right)^{-1}\cdot H_{\omega}\cdot \left(\hat{\epsilon}_{\delta}(\omega^*)  - \hat{\epsilon}(\omega^*) \right)  \\
&- \delta \cdot \left(H_{\omega}+ H_{\omega} \cdot\frac{\mathrm{d}  \hat{{\epsilon}}_{\delta}\left(\omega_{\delta}\right)}{\mathrm{d} \omega}\right)^{-1}\cdot \nabla L^k_S\left(\omega^*+\hat{\epsilon}(\omega^*)\right).
\end{align*}
Then we can obtain the following equation:
\begin{align*}
    \left.\frac{\mathrm{d} \omega_{\delta}}{\mathrm{d}\delta}\right|_{\delta = 0}
      = & - \left(H_{\omega}+ H_{\omega} \cdot\frac{\mathrm{d} \hat{{\epsilon}}\left(\omega^*\right)}{\mathrm{d} \omega}\right)^{-1}\cdot H_{\omega}\cdot 
        \left.\frac{\mathrm{d}  \hat{\epsilon}_{\delta}(\omega^*)}{\mathrm{d}\delta}\right|_{\delta = 0} \\
        &-  \left(H_{\omega}+ H_{\omega} \cdot\frac{\mathrm{d} \hat{{\epsilon}}\left(\omega^*\right)}{\mathrm{d} \omega}\right)^{-1}\cdot \nabla L^k_S\left(\omega^*+\hat{\epsilon}(\omega^*)\right).
\end{align*}

To enhance the computation efficiency, we drop the first term, and obtain the SAM-HIF as
\begin{equation}
    \begin{split}
      \text{SAM-IF}=&  \left.\frac{\mathrm{d} \omega_{\delta}}{\mathrm{d}\delta}\right|_{\delta = 0}
      =-  \left(H_{\omega}+ H_{\omega} \cdot\frac{\mathrm{d} \hat{{\epsilon}}\left(\omega^*\right)}{\mathrm{d} \omega}\right)^{-1}\cdot \nabla L^k_S\left(\omega^*+\hat{\epsilon}(\omega^*)\right).
    \end{split}
\end{equation}
\end{proof}

\subsection{Computation Acceleration}
\begin{theorem}\label{app:IS}
    Given a validation dataset defined as $D_{val} = \{(x_t,y_t)\}_{t=1}^n$.
    Denote the SAM-retrained model after the removal of $(x_k,y_k)$ as $\omega^*_{-k}$, then 
    \begin{equation*}
    \begin{split}
           &\sum_{(x,y)\in D_{val}}\ell(x,y;\omega^*) - \sum_{(x,y)\in D_{val}}\ell(x,y;\omega^*_{-k}) \\
           \approx &\sum_{(x,y)\in D_{val}} \nabla \ell(x,y;\omega^*) \cdot \text{IF}(x,y)\triangleq \text{IS}(x,y), 
    \end{split}
    \end{equation*}
    where IF can be  SAM-HIF or SAM-GIF. We define the right hand of the above equation as the Influence Score (IS). 
    \end{theorem}

\begin{proof}
    
\begin{equation*}
    \begin{split}
&\sum_{(x,y)\in D_{val}}\ell(x,y;\omega^*) - \sum_{(x,y)\in D_{val}}\ell(x,y;\omega^*_{-k}) \\
\approx &\sum_{(x,y)\in D_{val}} \nabla \ell(x,y;\omega^*) \cdot \left(\omega^* - \omega^*_{-k}\right)\\
=&\sum_{(x,y)\in D_{val}} \nabla \ell(x,y;\omega^*) \cdot \text{IF}(x,y)\triangleq \text{IS}(x,y).
    \end{split}
    \end{equation*}

\end{proof}

\section{Omitted Algorithms}

\begin{algorithm}
\caption{SAM-IF with gradient trajectory\label{alg:3}}
\begin{algorithmic}[1]
    \STATE {\bf Input:} \\
    {\bf data:} 
    Training Dataset $S = \{(x_i, y_i\}_{i=1}^n$, data point $({x}_k, {y}_k)$ to be evaluated.\\
    {\bf Parameter: } Parameter checkpoint set $\Omega = \{\omega_c\}_{c=1}^s$, learning rate $\eta_c$ at step $c$. \\
%     {\bf Definition: }
%     \begin{align}
%     \hat{\epsilon}(\omega) & = \rho \cdot \text{sign}(\nabla_w  L_{S}(\omega, {0}) )\cdot\frac{|\nabla_w  L_{S}(\omega, {0})  |^{q-1}}{\left( \|\nabla_w  L_{S}(\omega, {0}) \|_q^q \right)^{1/p}}
% \end{align}
    % {\bf Hyper parameter ρ\rho}, which is the norm constraints of the perturbation
\STATE Compute the influence of $(x_k, y_k)$ in the $c$-th checkpoint as 
\begin{equation*}
    \text{IF}_c = \eta_c \cdot \nabla L_S^k(\omega_{c-1} + \hat{\epsilon}(\omega_{c-1}))
\end{equation*}
\STATE Sum up to obtain the final influence as
\begin{equation*}
    \text{SAM-IF}_{\text{Step}} = \sum_{c=1}^s\text{IF}_c
\end{equation*}
\STATE {\bf Return: }$\text{SAM-IF}_{\text{Step}}$. Final Parameter $\omega^- = \omega^* - \text{SAM-IF}$.
\end{algorithmic}
\end{algorithm}

\begin{algorithm}
\caption{Simple SAM-IF\label{alg:simple}}
\begin{algorithmic}[1]
    \STATE {\bf Input:} \\
    {\bf Data:} 
    Training Dataset $S = \{(x_i, y_i\}_{i=1}^n$, data point $({x}_k, {y}_k)$ to be evaluated.\\
    {\bf Parameter: } 
    The learned SAM-parameter $\omega^*$, the learned best perturbation $\hat{\epsilon}$.
 \STATE Compute $T^k$ as:
 \begin{equation*}
     T^k = \nabla L^k_S\left(\omega^*+ \hat{\epsilon}\right)
 \end{equation*}
 \STATE Define the Hessian matrix $H$ as:
 \begin{equation*}
    \nabla^2 L_S\left(\omega^*+\hat{\epsilon}\right)
 \end{equation*}
\STATE Use EK-FAC to compute the Hessian-vector product of $H$ and $T^k$:
\begin{equation*}
    \text{SAM-IF} = -H^{-1}\cdot T^k
\end{equation*}
\STATE Obtain the estimated parameter by $$\omega^- = \omega^* -\text{SAM-IF}.$$
\STATE {\bf Return: }$\text{SAM-IF}$, $\omega^-$.
\end{algorithmic}
\end{algorithm}

\begin{algorithm}
\caption{SAM-IF with total Hessian\label{alg:complete}}
\begin{algorithmic}[1]
    \STATE {\bf Input:} \\
    {\bf Data:} 
    Training Dataset $S = \{(x_i, y_i\}_{i=1}^n$, data point $({x}_k, {y}_k)$ to be evaluated.\\
    {\bf Parameter: } 
    The learned SAM-parameter $\omega^*$, the learned best perturbation $\hat{\epsilon}$.
 \STATE Compute $T^k$ as:
 \begin{equation*}
     T^k = \nabla L^k_S\left(\omega^* + \hat{\epsilon}\right)
 \end{equation*}
 \STATE Define the gradient of $\epsilon$ as
 $$C^{\epsilon} = \nabla_{\omega}\frac{|\nabla_w  L_{S}(\omega^*)  |^{q-1}}{\left( \|\nabla_w  L_{S}(\omega^*) \|_q^q \right)^{1/p}}$$
 \STATE Give the definition of $H$ as
\begin{align*}
 H =&    \nabla^2 L_S\left(\omega^*+\hat{\epsilon}\right) + \nabla^2 L_S\left(\omega^*+ \hat{\epsilon}\right) \cdot  C^{\epsilon}.
\end{align*}
 \STATE $j\xleftarrow{}1$\\
 $T^k\xleftarrow{}I_0$
\IF{$\|I_j - I_{j-1}\|_1 > \zeta$}
\STATE Use EK-FAC to compute the Hessian-vector product $S_j\triangleq H\cdot  I_j$.
\STATE Compute $I_{j+1}$ by
\begin{equation*}
I_{j+1}= I_{j} -  \delta\cdot I_{j} + S_j +  T^k
\end{equation*}
\STATE $j\xleftarrow{}j+1$
\ENDIF
\STATE $\text{SAM-IF} \leftarrow I_{j+1}$.
\STATE Obtain the estimated parameter by $\omega^- = \omega^* -\text{SAM-IF}.$
\STATE {\bf Return: }$\text{SAM-IF}$. $\omega^-$.
\end{algorithmic}
\end{algorithm}

\clearpage

\section{Additional Experiment Results}
This section presents the additional experimental results for completeness and detailed analysis.

\subsection{Results of Efficiency and Accuracy}\label{app:exp:efficiency}
We conduct experiments on MNIST and HAM with Wide-Resnet as the model.
\begin{table*}[ht]
\centering
\caption{Performance comparison on MNIST and HAM.}
\label{tab:results}
\resizebox{0.8\linewidth}{!}{
\begin{tabular}{lcccc}
\toprule
    \multirow{3}{*}{\textbf{Method}} & \multicolumn{2}{c}{\textbf{Mnist}} & \multicolumn{2}{c}{\textbf{HAM}}\\
    \cmidrule(r){2-3} \cmidrule(r){4-5} 
    & \textbf{Accuracy} & \textbf{RT (second)} & \textbf{Accuracy} & \textbf{RT (second)} \\
\midrule
    Retrain & 0.9927$\pm$0.0005 & 2304.24$\pm$7.91 & 0.7254$\pm$0.02 & 2300.00$\pm$10.00\\
    SAM-HIF(Fast)  & 0.9876$\pm$0.0013 &  8.0698$\pm$2.91 & 0.7130$\pm$0.0150 & 15.00$\pm$2.00\\
    SAM-HIF  & 0.9880$\pm$0.0011 & 18.2321$\pm$1.91 & 0.7212$\pm$0.0120 & 52.20$\pm$3.00\\
    SAM-GIF  & 0.9884$\pm$0.0007 & 2.8212$\pm$1.91 & 0.7350$\pm$0.0110 & 4.320$\pm$1.50 \\
\bottomrule
\end{tabular}}
\vspace{-7pt}
\end{table*}
\subsection{Results of Identifying Harmful Data}
We conduct additional harmful data identification experiments on the MNIST, CIFAR-100, and Mini-ImageNet datasets. The results are listed as follows. Figures \ref{app:id_h_cifa-100}, \ref{app:id_h_mini}, and \ref{app:id_h_minist} illustrate the effectiveness of the SAM-GIF algorithm in detecting and removing harmful data.
The superiority of our method is demonstrated across multiple datasets. We observe that, compared to random removal, the harmful detection rate of the SAM-GIF algorithm reaches approximately 80\% at a removal rate of 0.4. Moreover, by removing harmful data, the model's accuracy gradually improves.

\begin{figure}[htbp]
    \centering
    \begin{subfigure}{0.35\linewidth}
        \centering
        \includegraphics[width=\linewidth]{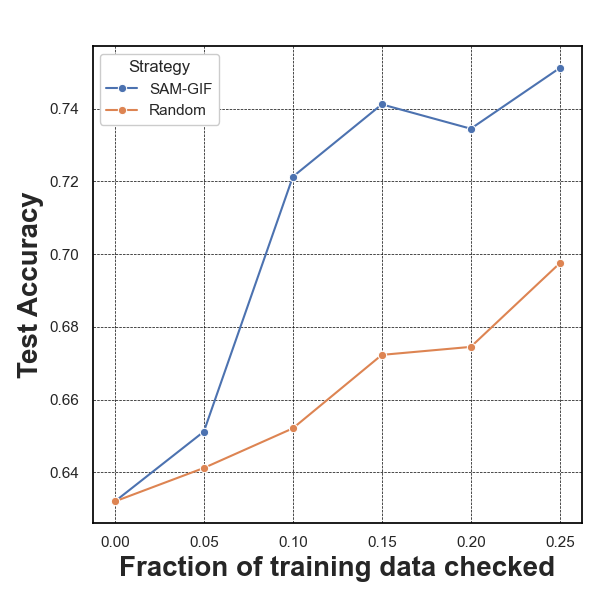}
        \caption{Test Accuracy}
        \label{fig:hdr1}
    \end{subfigure}%
    \begin{subfigure}{0.35\linewidth}
        \centering
        \includegraphics[width=\linewidth]{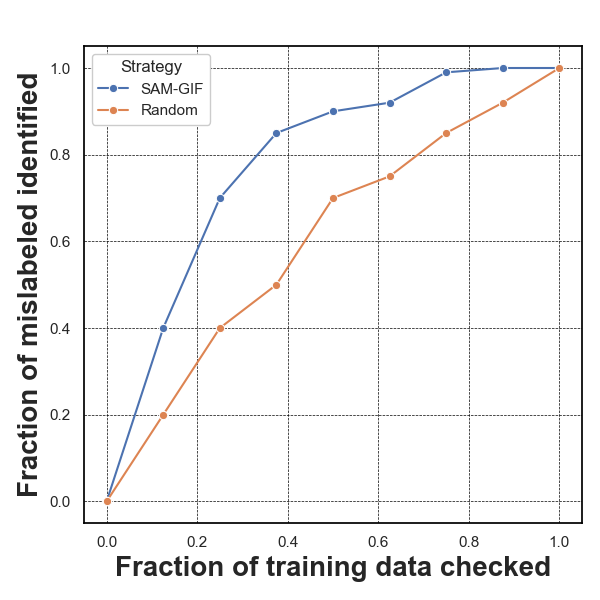}
        \caption{Harmful Data Removal}
        \label{fig:accuracy1}
    \end{subfigure}
    \caption{Harmful data removal experiment on CIFAR-100 dataset. IS: using the influence score to determine which sample to remove. Random: randomly removing tasks.}
    \label{app:id_h_cifa-100}
\vspace{-12pt}
\end{figure}

\begin{figure}[htbp]
    \centering
    \begin{subfigure}{0.35\linewidth}
        \centering
        \includegraphics[width=\linewidth]{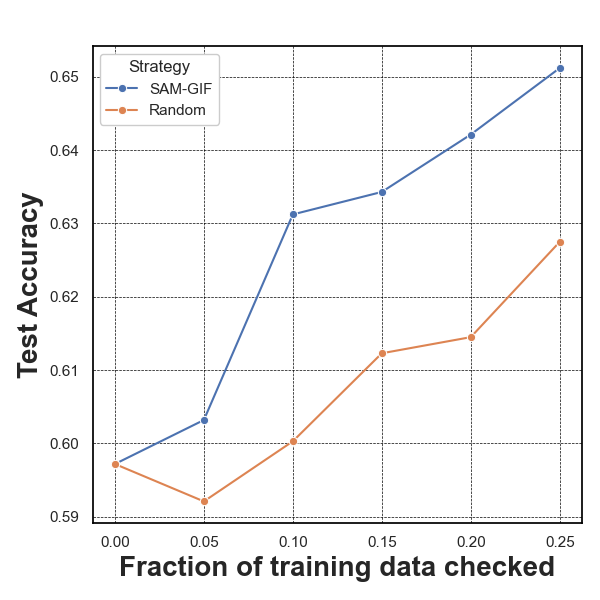}
        \caption{Test Accuracy}
        \label{fig:hdr2}
    \end{subfigure}%
    \begin{subfigure}{0.35\linewidth}
        \centering
        \includegraphics[width=\linewidth]{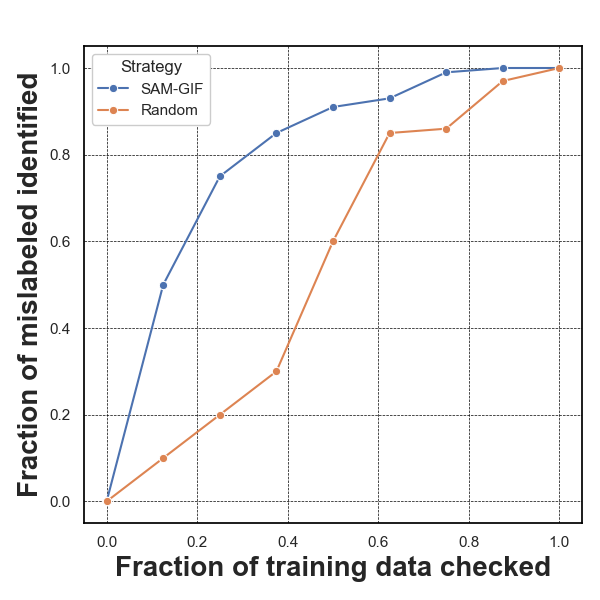}
        \caption{Harmful Data Removal}
        \label{fig:accuracy2}
    \end{subfigure}
    \caption{Harmful data removal experiment on Mini-ImageNet dataset. IS: using the influence score to determine which sample to remove. Random: randomly removing tasks.}
    \label{app:id_h_mini}
\vspace{-12pt}
\end{figure}

\begin{figure}[htbp]
    \centering
    \begin{subfigure}{0.35\linewidth}
        \centering
        \includegraphics[width=\linewidth]{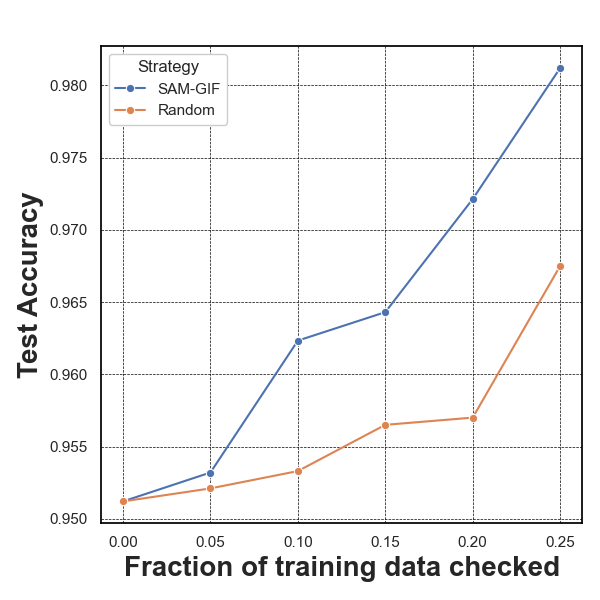}
        \caption{Test Accuracy}
        \label{fig:hdr3}
    \end{subfigure}%
    \begin{subfigure}{0.35\linewidth}
        \centering
        \includegraphics[width=\linewidth]{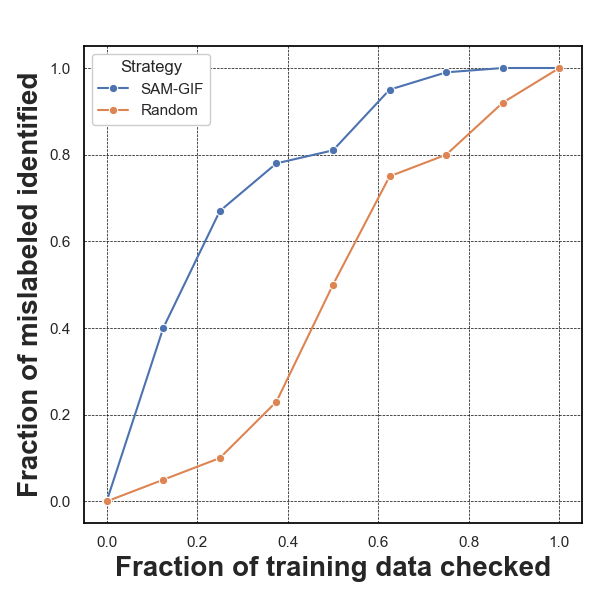}
        \caption{Harmful Data Removal}
        \label{fig:accuracy3}
    \end{subfigure}
    \caption{Harmful data removal experiment on MNIST dataset. IS: using the influence score to determine which sample to remove. Random: randomly removing tasks }
    \label{app:id_h_minist}
\vspace{-12pt}
\end{figure}

\subsection{Additional Results on Interpretability}
Figures~\ref{app:inter_minist}, \ref{app:inter_cifar100}, and \ref{app:inter_imagenet} present additional visualization results of the error prediction tracing process on the MNIST, CIFAR-100, and Mini-ImageNet datasets, respectively. The first row displays examples of misclassified test samples, the second row shows the most influential training data for classifying these samples, and the third row presents the most harmful training data for these classifications. These visualization results allow us to trace the outcomes of the error prediction process effectively.

\begin{figure}[ht]
    \centering
    \includegraphics[width=0.75\linewidth]{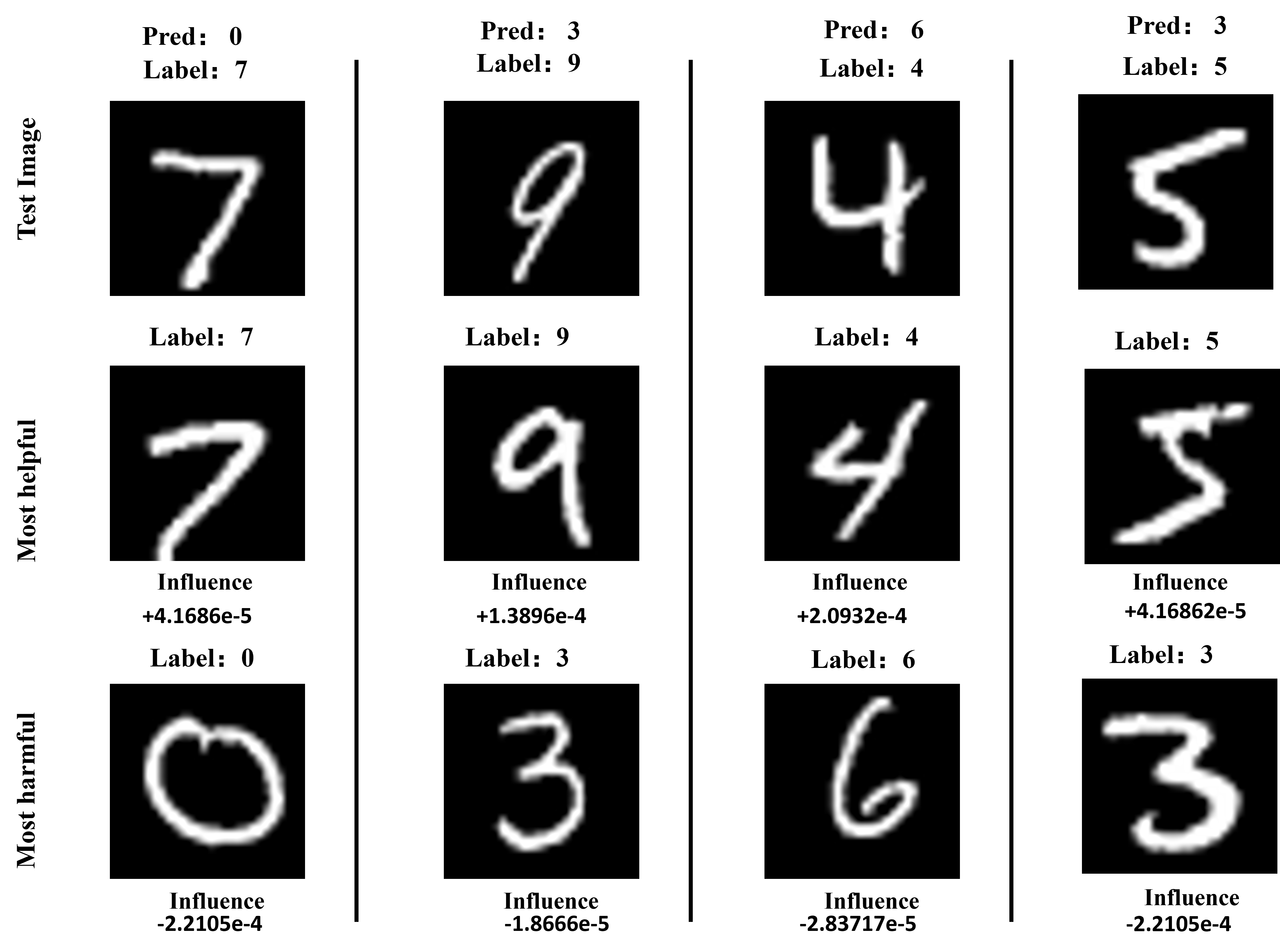}
    \caption{The most helpful and harmful training data tracked by misclassified data on MNIST dataset}
    \label{app:inter_minist}
\end{figure}

\begin{figure}[ht]
    \centering
    \includegraphics[width=0.75\linewidth]{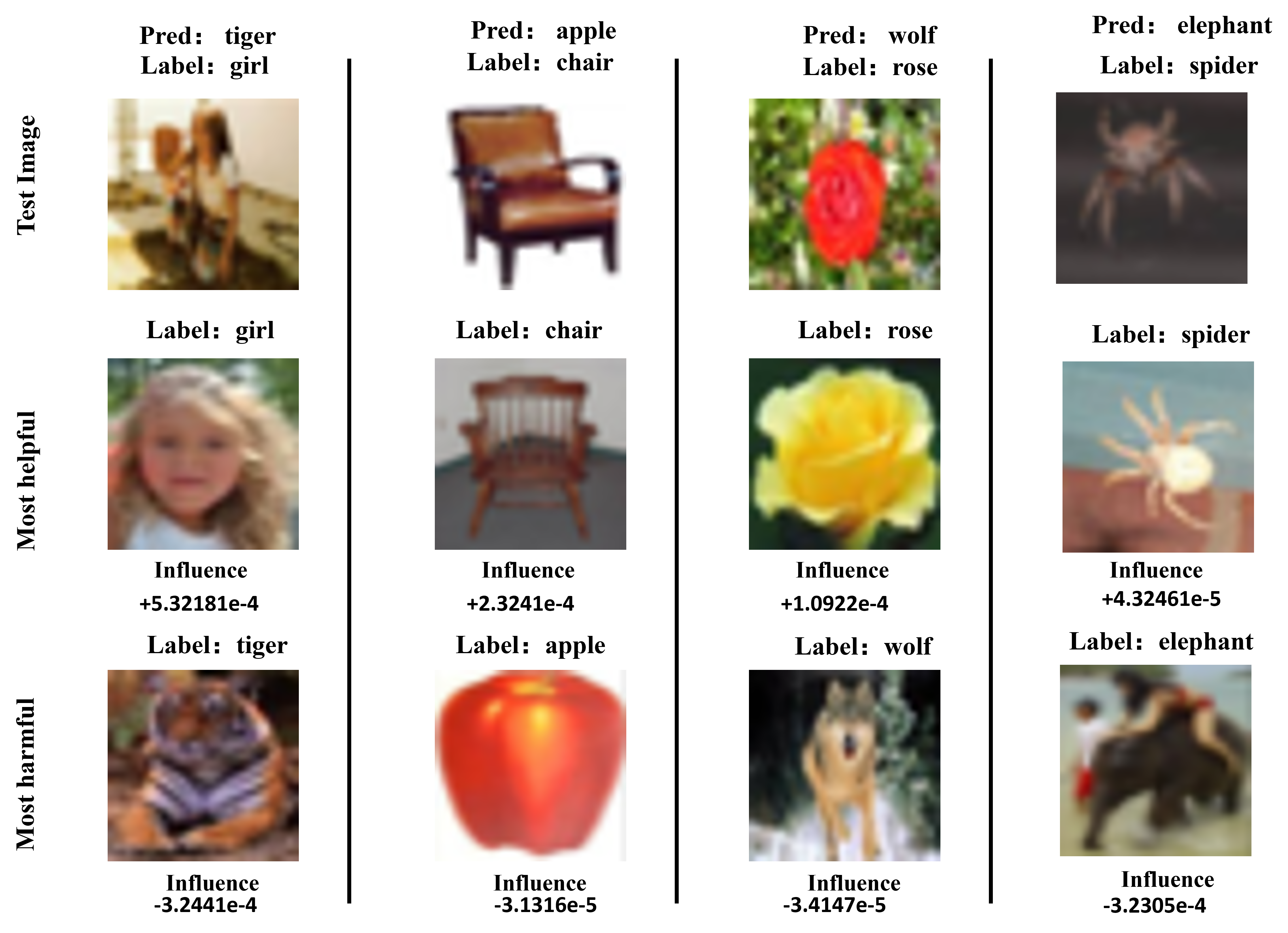}
    \caption{The most helpful and harmful training data tracked by misclassified data on CIFAR-100 dataset}
    \label{app:inter_cifar100}
\end{figure}

\begin{figure}[ht]
    \centering
    \includegraphics[width=0.75\linewidth]{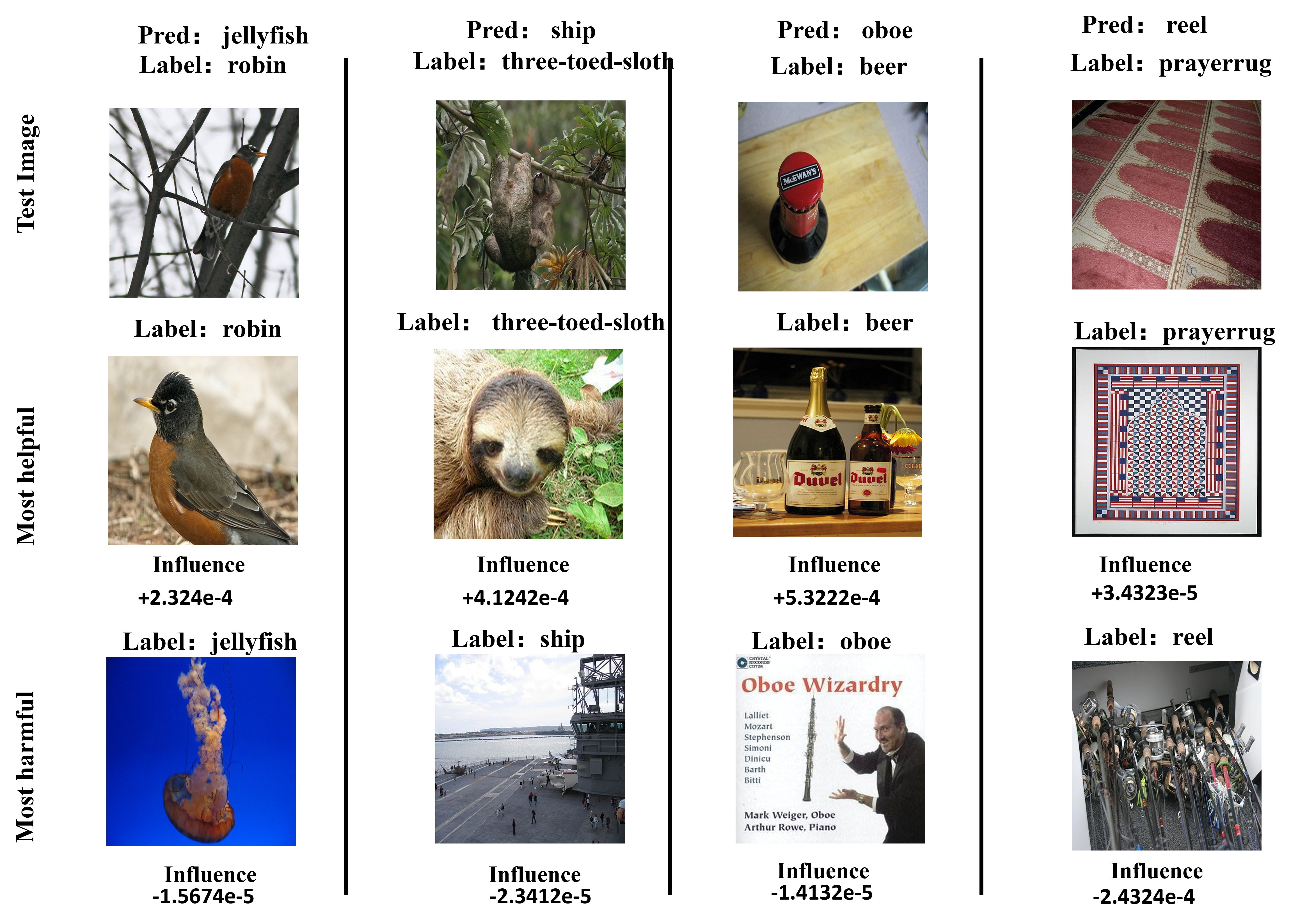}
    \caption{The most helpful and harmful training data tracked by misclassified data on Mini-ImageNet dataset}
    \label{app:inter_imagenet}
\end{figure}

\subsection{Additional Ablation Study}
Figure \ref{app:ablation} shows the ablation study results of SAM-GIF. We primarily tested the impact of different numbers of checkpoint weights on the SAM-GIF algorithm using the CIFAR-10 dataset. From the figure, we can observe that as the number of checkpoints increases, the accuracy of SAM-GIF becomes closer to that of retraining. When the number of checkpoints is 10, the accuracy of SAM-GIF is 0.9497, while the retraining accuracy is 0.9517. At the same time, as the number of checkpoints increases, the running time of SAM-GIF also increases.
\begin{figure}[ht]
    \centering
    \begin{subfigure}{0.35\linewidth}
        \centering
        \includegraphics[width=\linewidth]{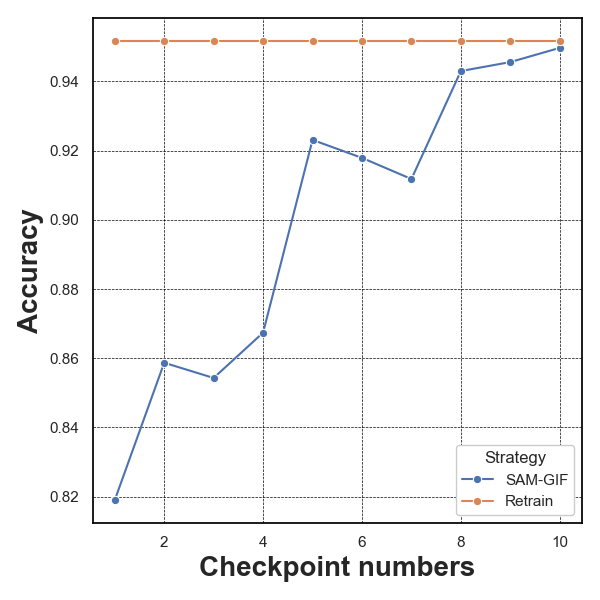}
        \caption{Accuracy of SAM-GIF}
        \label{fig:hdr4}
    \end{subfigure}%
    \begin{subfigure}{0.35\linewidth}
        \centering
        \includegraphics[width=\linewidth]{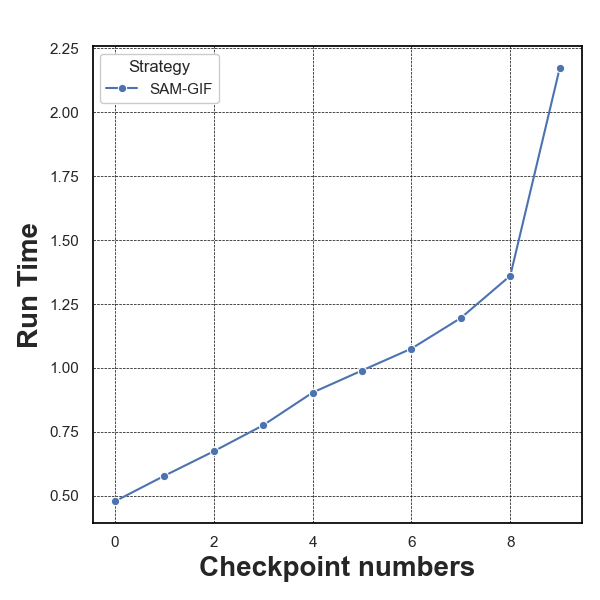}
        \caption{Runtime of SAM-GIF}
        \label{fig:accuracy4}
    \end{subfigure}
    \caption{Ablation study of SAM-GIF on Cifar10 dataset}
    \label{app:ablation}
\vspace{-12pt}
\end{figure}

\end{document}